\documentclass[twocolumn]{article}

\usepackage[english]{babel}

\usepackage[T1]{fontenc}

\usepackage{float}
\usepackage{tabularx}
\usepackage{booktabs}
\usepackage[most]{tcolorbox}
\usepackage{subcaption} 

\usepackage{latexsym}
\usepackage{amssymb}
\usepackage{amsmath}
\usepackage{amsfonts}
\usepackage{enumitem}
\usepackage{bbm}
\usepackage{dsfont}
\usepackage[normalem]{ulem}

\usepackage{natbib}
\usepackage{algorithm}
\usepackage{algpseudocode}
\algnewcommand{\IfThen}[2]{
  \State \algorithmicif\ #1\ \algorithmicthen\ #2\ \algorithmicend\ \algorithmicif}
\usepackage{eqparbox}

\usepackage{tikz}
\usetikzlibrary{positioning}
\usepackage{pgfplots}
\pgfplotsset{compat=1.18}
\usepgfplotslibrary{fillbetween}
\pgfkeys{/pgf/number format/.cd ,precision=2,sci}
\newcommand\convert[1]{\pgfmathprintnumber{#1}}
\newcommand\convertBold[1]{$\boldsymbol{\pgfmathprintnumber[assume math mode=true]{#1}}$}
\usepackage{lineno}

\DeclareMathOperator*{\argmin}{arg\,min}
\DeclareMathOperator*{\argmax}{arg\,max}
\makeatletter
\DeclareRobustCommand*\cal{\@fontswitch\relax\mathcal}
\makeatother

\usepackage{amsthm}
\usepackage{thmtools} 
\usepackage{thm-restate}

\theoremstyle{plain}
\newtheorem{theorem}{Theorem}[section]

\newtheorem{lemma}[theorem]{Lemma}
\newtheorem{corollary}[theorem]{Corollary}

\newtheorem{definition}[theorem]{Definition}

\newtheorem{remark}[theorem]{Remark}

\newcommand{\eqdef}{\triangleq}

\usepackage{listings}
\usepackage{xcolor}

\usepackage{pgf}

\definecolor{codegreen}{rgb}{0,0.6,0}
\definecolor{codegray}{rgb}{0.5,0.5,0.5}
\definecolor{codepurple}{rgb}{0.58,0,0.82}
\definecolor{backcolour}{rgb}{0.95,0.95,0.92}

\lstdefinestyle{mystyle}{
    backgroundcolor=\color{backcolour},   
    commentstyle=\color{codegreen},
    keywordstyle=\color{magenta},
    numberstyle=\tiny\color{codegray},
    stringstyle=\color{codepurple},
    basicstyle=\ttfamily\footnotesize,
    breakatwhitespace=false,         
    breaklines=true,                 
    captionpos=b,                    
    keepspaces=true,                 
    numbers=left,                    
    numbersep=5pt,                  
    showspaces=false,                
    showstringspaces=false,
    showtabs=false,                  
    tabsize=2
}

\usepackage{hyperref}
\usepackage[noabbrev]{cleveref}

\usepackage{multirow}

\def\cL{L}
\def\cP{P}
\def\hmu{\hat{\mu}}
\def\bmu{\bar{\mu}}
\def\cO{\mathcal{O}}
\def\cN{\mathcal{N}}
\def\cI{\mathcal{I}}
\def\cB{\mathcal{B}}
\def\cT{\mathcal{T}}
\def\cM{\mathcal{M}}
\def\rg{D}

\def\ouralgo{\texttt{MAUB}}

\def\vmu{\boldsymbol{\mu}}
\def\vbmu{{\boldsymbol{\bar \mu}}}
\def\vhmu{{\boldsymbol{\hat \mu}}}

\usetikzlibrary{shapes.geometric}

\newcommand\modif[1]{{\color{black}#1}}

\title{{\bf Efficient Matroid Bandit Linear Optimization Leveraging Unimodality}}

\author{Aurélien Delage$^1$, 
Romaric Gaudel$^1$
\\
$^1$ Univ. Rennes, \texttt{Inria}, \texttt{CNRS} \texttt{IRISA} - \texttt{UMR} 6074, F35000 Rennes, France
\\ \texttt{prenom.nom@irisa.fr}
}

\begin{document}

\maketitle

\begin{abstract}
    We study the combinatorial semi-bandit problem under matroid constraints.
    The regret achieved by recent approaches is optimal, in the sense that it matches the lower bound.
    Yet, time complexity remains an issue for large matroids or for matroids with costly membership oracles ({\em e.g.} online recommendation that ensures diversity).
    This paper sheds a new light on the matroid semi-bandit problem by exploiting its underlying unimodal structure.
    We demonstrate that, with negligible loss in regret, the number of iterations involving the membership oracle can be limited to $\cO(\log\log T)$. 
    This results in an overall improved time complexity of the learning process. 
    Experiments conducted on various matroid benchmarks show (i) no loss in regret compared to state-of-the-art approaches; and (ii) reduced time complexity and number of calls to the membership oracle.
\end{abstract}

\section{Introduction}

The semi-bandit framework models the interaction of an online learner that repeatedly selects one or multiple items to play from a known, and usually finite, groundset $E$. Stochastic feedback in terms of rewards is received for each selected item at each round. The goal is to maximize the sum of expected rewards, which is equivalent to minimizing the sum of regret for not selecting the optimal subset at each round.
Matroid bandit problems constrain the subsets of $E$ that can be played.
Applications range from learning routing networks \citep{kve-uai-14} and advertisement \citep{str-nips-09}, to task assignment \citep{che-colt-16}, learning maximum spanning trees \citep{pap-cc-98}, and leader-follower multi-agent planning \citep{cla-cdc-12,lin-cdc-11}.

The bandit algorithms designed for combinatorial problems with linear objective functions, such as \texttt{CUCB} \citep{che-icml-13} and \texttt{ESCB} \citep{com-nips-15} apply to matroid bandit problems. 
Their versions leveraging the strong matroid structure, respectively \texttt{OMM} \citep{kve-uai-14} and \texttt{KL-OSM} \citep{tal-aamas-16}, achieve better asymptotic regret, and \texttt{KL-OSM} is even optimal. 
However, time complexity still remains an issue, particularly for matroids with costly membership oracles \citep{tze-icml-24}.
The main bottleneck is the call to the {\em $\mathrm{greedy}$} algorithm at each iteration. 
The latter suffers a $\cO(|E|(\log|E|+\cT_{m}))$ time complexity, where $\cT_{m}$ is the time complexity of the membership oracle determining whether a set $I\cup\{x\}$ belongs to the matroid, given that $I \subset E$ does. 
Recently, \citet{tze-icml-24} proposed \texttt{FasterCUCB}, a first attempt to reduce the overall time complexity.
The authors approximate the greedy procedure using a dynamic maintenance of maximum-weight bases. 
Yet, this comes at the cost of a loss in regret; the introduction of domain-dependent update oracles; and an additional $\mathrm{polylog}(T)$ term in the per-round time complexity.

\paragraph{Contributions}
We present \texttt{MAUB}, a unimodal bandit algorithm tailored for matroid optimization. Our main contributions are:
\begin{enumerate}
    \item with negligible loss in regret, \ouralgo{} requires only $\cO(\log\log(T))$ calls to the matroid structure, leading to a reduced overall time complexity compared to \texttt{OMM}, \texttt{KL-OSM} and Faster\texttt{CUCB};
     \item \texttt{MAUB} does not require instance-specific matroid oracles, which makes it more general; and
     \item our alternative approach to matroid optimization  highlights the interest of unimodal structures for both the analysis of combinatorial bandit problems, and practical implementation of simple, yet efficient, learning algorithms.
\end{enumerate}

The article is organized as follows: Section~\ref{sec:relatedWork} presents the related work and Section~\ref{sec:background} gives the necessary background.
Next, we introduce our proposed algorithm, namely \ouralgo{}, in Section~\ref{sec:MAUB}.
Its theoretical guarantees are studied in Section~\ref{sec:MAUBAnalysis}, and Section~\ref{sec:experiments} illustrates its empirical behavior in numerical experiments.

\section{Related Work}
\label{sec:relatedWork}

\paragraph{Unimodal Bandit}

Unimodality considers graph-structured bandit problems. It is assumed that unless the current arm is optimal, there always exists a strictly better one in its neighborhood.  
It defines an important subclass of the general bandit problem \citep{cop-cdc-09,yu-icml-11, com-icml-14, tri-alt-20}.
To handle this class of problems, \cite{com-icml-14}'s algorithm, namely \texttt{OSUB}, splits the optimization problem at each round into two parts: (i) determining the best arm so far; and (ii) deciding what to play to ensure optimism.
The first problem is solved by examining {\em mean-statistics}, and the second one only looks at arms within the neighborhood of the current best arm. 
\cite{tri-alt-20, gau-icml-21, gau-icml-22} extended \texttt{OSUB} to the semi-bandit setting to handle their target application, multiple-play online recommendation systems (\texttt{ORS}).

The restricted exploration set has been leveraged by \cite{tri-alt-20, gau-icml-21, gau-icml-22} to design regret-optimal algorithms, by aligning neighborhood structures with the gap-dependent terms appearing in the logarithmic regret bound. 
While they did it to simplify algorithm design and regret analysis, in current paper we demonstrate that unimodality also path the way toward  the reduction of computational costs.
This reduction relies on two key factors: first, the neighborhood rarely changes, which minimizes the need for combinatorial optimization; and second, most of the time, decisions are restricted to a small neighborhood, keeping the per-step computation low.

\paragraph{Optimal Regret in Matroid Linear Optimization}

Table~\ref{tab:regretAndTimeAlgs} summarizes regret upper-bounds and time complexities of state-of-the-art matroid bandit algorithms. 
Both optimistic matroid maximization (\texttt{OMM}) \citep{kve-uai-14} and \texttt{KL}-based efficient sampling for matroids (\texttt{KL-OSM}) \citep{tal-aamas-16} rely on repeated calls to the greedy algorithm that computes the arm $B^*$ with the highest expected value according to {\em optimistic statistics}.
\texttt{OMM} corresponds to the direct application of \texttt{CUCB} \citep{che-icml-13} to matroids.
It uses \texttt{UCB}-like bonuses, and achieves a $\cO\left(\frac{|E|-D}{\Delta_{\mathrm{min}}}\log(T)\right)$ regret, where $D$ is the rank of the matroid, and $\Delta_{\mathrm{min}}$ is the minimum gap between the mean value of an element in the optimal arm and any element outside it.
\texttt{KL-OSM} addresses an extraneous multiplicative factor in this regret by replacing \texttt{UCB} bonuses with \texttt{KL} indices, yielding an optimal algorithm that matches the lower bound of \cite{kve-uai-14}.
However, reducing regret comes at a price: it limits applicability to reward distributions on $[0,1]$ and increases computational cost.

\paragraph{Approximation to Reduce Time Complexity}
Having these regret-optimal algorithms, the remaining question is computation complexity.
%
%
\cite{tze-icml-24} introduced  Faster\texttt{CUCB}, an algorithm focusing on this aspect. 
The approach maintains a maximum-weight basis, which is updated across iterations.
Calls to the membership oracle are thus replaced with calls to an update oracle, for which class-dependent efficient implementations exist. 
The cost, however, is the introduction of a multiplicative constant in the overall regret, and the per-round time complexity suffers an additional multiplicative $\mathrm{polylog}(T)$ factor. 
Furthermore, while the-round time complexity of Faster\texttt{CUCB} is sublinear in $|E|$ for various classes of matroids, it also contains a $\mathrm{polylog}(T)$ factor, which makes it slower than \ouralgo{} on the long run (as soon as $D\cT_u\mathrm{polylog}(T)>|E|$). 

For the sake of completeness,
\cite{per-icml-19} study the matroid semi-bandit problem when there is a submodular function to optimize. The rationale is that such function is either the primary objective, or arises from the design of the optimistic term ({\em e.g. } \texttt{ESCB} \citep{deg-nips-16}).
Solving the corresponding combinatorial problem at each iteration would be too prohibitive, so they (approximately) solve it  using either local search or a variant of the greedy algorithm. 
Yet, the time complexity is at least $\cO(D|E|)$, and each submodular maximization does not decrease the number of calls to the membership oracle compared to \texttt{OMM}.
In current paper we limit ourselves to linear objective for which such fancy optimism is not required to get regret-optimal algorithms.

\begin{table}
\centering
\caption{Overall regret and time complexity achieved by \texttt{OMM}, \texttt{KL-OSM}, Faster\texttt{CUCB} and \texttt{MAUB} for the matroid bandit problem. $\cT_u$ is the time complexity of updating maximum-weight bases in Faster\texttt{CUCB}. \citet[Section~3]{tze-icml-24} detail $\cT_u$ for uniform, graphic, partition, and transversal matroids. The mapping $\sigma$ is defined in Section~\ref{sec:unimodality}, and $\mathrm{poly}$ is a polynomial.}
\label{tab:regretAndTimeAlgs}
\resizebox{\linewidth}{!}{
 \begin{tabular}{@{}lll@{}}
    \toprule
    & Regret & Time Complexity (in $\cO()$)\\
      \midrule
     \texttt{CUCB}/\texttt{OMM} \citenum{kve-uai-14}    & $\sum_{e \notin B^*} \frac{16\cdot\log T}{\min_{i \in B^*, \mu_i>\mu_e}\mu_i-\mu_e}$  & $|E|(\log |E| + \cT_m)T$\\
     \texttt{KL-OSM} \citenum{tal-aamas-16} & $\sum_{e \notin B^*} \frac{(\mu_e - \mu_{\sigma(e)}) \log T}{\mathrm{kl}(\mu_e,\mu_{\sigma(e)})}$ & $|E|(\log |E| + \cT_m)T$\\
     Faster\texttt{CUCB} \citenum{tze-icml-24} &$\sum_{e \notin B^*} \frac{12\max_i|\mathrm{supp}(\mu_i)|\cdot\log T}{\min_{i \in B^*, \mu_i>\mu_e}\mu_i-\mu_e}$ & $ D\cT_uT\mathrm{polylog}\ T$ \\
     \texttt{MAUB} (Th.\ref{cor:timeComplexity}) & $\sum_{e \notin B^*} \frac{8\cdot\log T}{\mu_{\sigma(e) -\mu_e}}$ &\resizebox{0.65\linewidth}{!}{$|E|T+ \binom{|E|}{\rg}\mathrm{poly}(|E|,D)\cT_m\log\log T$}\\
    \bottomrule
  \end{tabular}
}
\end{table}

\section{Background}
\label{sec:background}

The following provides background and notations on general combinatorial bandit problems, the matroid structure, and the corresponding greedy algorithm.

{
We use bold symbols to denote collections ({\em e.g.} $\vmu$ stands for $(\mu_e)_{e \in E}$).
Also, for any set $I$ and element $e$, we write $I+e$ (resp. $I-e$) for $I\cup\{e\}$ (resp. $I\setminus\{e\}$). 
}

\paragraph{Combinatorial Semi-Bandit Setting}
We consider a combinatorial semi-bandit problem in which arms are subsets of a groundset $E$ of elementary actions. 
For any $e$ in $E$, $X_e(t)$ is the random variable associated with the reward of $e$ at time step $t$.
All random variables $(X_e(t))_{t\geq 1}$ are assumed i.i.d., with unknown mean $\mu_e\in(0,\infty)$.
We further assume that $\forall x,y \in E,\ \mu_x\neq\mu_y$. 
For any subset $I\subseteq E$, we define at time step $t$ $X_I(t)\eqdef\sum_{e\in I}X_e(t)$, and the corresponding expectation $\mu_I \eqdef\sum_{e\in I}\mu_e$. 
Random variables $(X_I(t))_{I\subseteq E}$ are consequently correlated. 

Semi-bandit interactions are done as follows: there is a set of allowed combinations $\cI\subseteq 2^E$ such that at each time step $t$, the player plays a subset $\cP(t)\in \cI$, observes $(X_e(t))_{e\in \cP(t)}$ and is rewarded $X_{\cP(t)}(t)$.
In this scenario, learning algorithms aim
at minimizing the \emph{(pseudo-)regret}
\begin{align}    
\label{def:regret}
R(t) \eqdef \sum_{s \leq t} \max_{I\in\cI}\mu_I - \mu_{\cP(s)}.
\end{align}

\paragraph{Matroids}
Matroids are constraining the set $\cI$.

\begin{definition}[Matroid, groundset, independent subsets, rank, bases]
A {\em matroid} $\cM$ is a pair $(E,\cI)$, where $E$ is a set of items, called the \emph{groundset}, and $\cI\subseteq 2^E$ is a subset of the powerset of $E$. Subsets $I$ in $\cI$ are said {\em independent}.

The matroid structure requires that (i) $\emptyset \in \cI$, (ii) {\em hereditary property}: every subset of an independent set is also independent, and (iii) {\em augmentation property}: for all $I,J$ in $\cI, |I|=|J|+1$ implies that there exists $e\in I\setminus J$ such that $J+e\in\cI$.

Axioms (ii) and (iii) imply that all maximal sets (for the inclusion order) have equal size, which defines the {\em rank} $\rg$ of $\cM$. These sets are the \emph{bases} $B\in\cB$ of $\cM$. 

\end{definition}

From these axioms, it follows the \emph{basis exchange property},
which underpins the unimodal graph explored by \ouralgo{} to identify an optimal independent set.

\begin{lemma}[basis exchange property \protect{\cite[Th.~10.2 ~p.177]{sch-cwi-03}}, \protect{\cite[Prop.~1]{tal-aamas-16}}]
\label{lem:exchange}
For any matroid $M=(E,\cI)$, it holds that:
    $$\forall X,Y\in\cB \implies \forall x\in X \setminus Y,\ \exists y \in Y \setminus X,\ X-x+y \in \cB.$$
\end{lemma}

\paragraph{Linear Optimization in Matroids}

All means $\mu_e$ for $e$ in $E$ being non-negative, and since $\cB$ is finite, an independent set $B^*$ with highest value must be a base (namely $B^*\in\cB$). In the following, we thus restrict the arms of the bandit problem to bases $B \in \cB$.
We further assume that $B^*$ is uniquely defined.

If values $\vmu$ where known, one could compute the optimal base $B^*$ through the {\em greedy} algorithm (Algorithm~\ref{alg:greedy}). In essence, greedy iteratively constructs an optimal basis $B^*$ by considering elements $e_i \in E$ in decreasing order. At a step $i$, $e_i$ is added to $B^*$ if and only if it does not make $B^*$ dependent ({\em i.e.} $B^*+e_i\in\cI$). 

\begin{algorithm}
    \caption{Greedy Algorithm \citep{edm-mp-71}}
\label{alg:greedy}
\begin{algorithmic}
\Require element values $\vmu$
\State $B^*\gets\emptyset$
\State Find an ordering $\mu_{e_1}\geqslant\dots\geqslant\mu_{e_{|E|}}$
\For{$i=1,\dots,K$}
\IfThen{$B^*+e_i\in\cI$}{$B^*\gets B^*+e_i$}
\EndFor
\State \textbf{Return} $B^*$
\end{algorithmic}
\end{algorithm}

Each independence property is tested by querying a membership oracle with the subset $B^*+e_i$.
Hence the $\cO\left(T|E|(\log |E| + \cT_m)\right)$ time complexity of \texttt{OMM} and \texttt{KL-OSM}, which are calling greedy at each time step (on optimistic estimates of $\vmu$).

\section{Unimodal Approach to Learn in Matroids}
\label{sec:MAUB}

Unimodal bandits enable decoupling the optimistic selection rule from the combinatorial optimization over the matroid.
Concretely, optimization is performed on empirical means and only when the leader changes.
Since leader changes are rare, this results in very few membership-oracle queries and thus low computational cost.
The arm to play is then chosen optimistically within a small candidate set of neighbors.

The following formally describes the underlying unimodal structure in matroid bandit problems as well as our unimodal bandit algorithm that leverages it.

\subsection{Unimodality}
\label{sec:unimodality}

Let us first define the graph associated to a matroid and then prove that it is unimodal.
Let $M=(E,\cI)$ be a matroid
of bases $\cB$, whose elements are associated to expectations $\boldsymbol{\mu}$. Let $B \in \cB$ and define the mapping $\sigma_{B,\boldsymbol{\mu}}: E \setminus B \to B$ such that $\forall e \in E \setminus B,\ \sigma_{B,\boldsymbol{\mu}}(e)=\argmin_{\{x \in B: B-x+e\in\cB\}} \mu_x$.
$\sigma_{B,\boldsymbol{\mu}}$ maps the external elements $e$ to the element $x$ in $B$ with lowest value among those that can be swapped with $e$. 
For the sake of conciseness, we denote $\sigma$ the mapping $\sigma_{B^*,\vmu}$.

Using these mappings, we define the graph $G_{\vmu}=(S,A)$, where $S=\cB$, and the neighbors of a basis $B$ is $\cN_{B,\vmu} \eqdef \cup_{e \in E \setminus B} B-\sigma_{B, \vmu}(e)+e$. We below show that unimodality holds in $G_{\vmu}$.

\begin{restatable}[Unimodality \protect\footnotemark]{theorem}{unimodality}
    \label{th:unimodality}
    \footnotetext{\cite{tal-aamas-16} implicitly use this result, within the regret analysis of \texttt{KL-OSM}, when showing that sets $\mathcal{K}_i=\{l \in B^*, B^*-l+i\in\cB\}$ are non-empty, allowing to consider their minima $l_i$, which satisfies $\mu_{l_i}>\mu_i$.}
    
Let $M=(E,\cI)$ be a matroid of bases $\cB$, and let $\vmu$ denote the expectations on elements $E$. 
    For any basis $B\in\cB$ such that $\mu_B\neq\max_{I\in\cI}\mu_I$, there exists $B^+\in \mathcal{N}_{B,\vmu}$ with $ \mu_{B^+} > \mu_{B}$.
\end{restatable}

\begin{proof}
    Let $B$ be a basis of $M$ and
    denote $B^*$ the optimal basis.
    As $\mu_B\neq\max_{I\in\cI}\mu_I$, $B\neq B^*$ and by rank-property of matroids, $B\setminus B^*$ is non-empty.
    let $x = \argmin_{\tilde x \in B\setminus B^*}\mu_{\tilde x}$ be the element in $B\setminus B^*$ with lowest value.
    We shall prove that one can create a subset $B-x+\alpha$ (belonging to the neighborhood of $B$) with higher value than $B$.

    Using Lemma~\ref{lem:exchange}, $\exists \alpha \in B^*\setminus B$ such that  $B-x+\alpha \in \cal B$. 
    By definition of $x$, $\sigma_{B,\vmu}(\alpha)=x$, and therefore $B-x+\alpha\in\cN_{B,\vmu}$.

    It holds that $\mu_{B-x+\alpha} - \mu_B=\mu_\alpha - \mu_x$, hence we wonder the sign of $\mu_\alpha - \mu_x$.
    Assume $\mu_\alpha<\mu_x$.
    The matroid axiom (iii) implies that $\exists \beta \in B\setminus (B^*-\alpha),~B^*-\alpha+\beta\in\cal B$.
    But then, $\mu_\beta\geqslant\mu_x$ (by definition of $x$), so that $\mu_\beta\geqslant\mu_x>\mu_\alpha$.
    Consequently, $\mu_{B^*-\alpha+\beta}>\mu_{B^*}$, which is absurd. 
    Hence $\mu_\alpha\geqslant\mu_x$, and even  $\mu_\alpha>\mu_x$ as $\alpha \neq x$. Therefore $\mu_{B-x+\alpha}>\mu_B$, which concludes the proof.
\end{proof}

\subsection{\texttt{MAUB}}

We now introduce \texttt{MAUB}, which follows the unimodal bandit framework of \citet{com-icml-14}: it maintains a \emph{leader} (i.e., the best arm identified so far), and plays optimistically within its neighborhood.

In contrast to \texttt{OSUB}, the graph $G_{\hat\vmu(t)}$ used by \texttt{MAUB} (defined below) is not static and may fail to be unimodal with respect to the true---only partially known---optimization problem. Nevertheless, we establish in the regret analysis (see Lemma~\ref{lemma:incorrectNeighborhood}) that $G_{\hat\vmu(t)}$ most of the time contains the correct neighborhood.
Before moving on to a detailed description of \texttt{MAUB}, we define several statistics of interest.

\begin{definition}[Unimodal Bandit Statistics]
    We respectively denote by $\cL(t)$ and $\cP(t)$ the leader and the arm played at time step $t$.
    Additionally, let: 
    \begin{itemize}
	\item $l_B(t) \eqdef \sum_{s=1}^t \mathbb{E}[\mathds{1}_{\cL(s)=B}]$ be the number of times $B$ was leader up to time $t$;
        \item $\hmu_e(t)$ the average value of $(X_e(s))$ for all time steps $s \leq t$ at which $e$ was played.
    \end{itemize}
\end{definition}

Let us then define $\bmu_{B}(t,L) \eqdef \sum_{e\in B} [\hat{\mu}_e(t-1) + \sqrt{2\cdot\log(l_L(t))/N_e(t-1)}]$ be the optimistic estimator of $B$, where $N_e(t) = \sum_{s=0}^t \mathds{1}_{e\in \cP(t)}$ is the number of times $e$ was in a basis that was played at some time step $s\leq t$. $l_L(t)$ is a ``local time'', in comparison with the classical $\log(t)$ term in \texttt{UCB} bonuses. It indicates how much time was spent on the current leader, and is tighter than using total time spent to decide whether external elements $e\in E \setminus \cL(t)$ should be tested in comparison with elements of the leader.

A pseudocode of \texttt{MAUB} is given in Algorithm~\ref{alg:MAUB}. 
It repeatedly performs three steps: (i) if needed, compute the new current leader and its corresponding neighborhood (\Crefrange{alg:line:testIfNewLeaderNeeded}{MAUB:line:endUpdateLeaderAndNeighbors}); (ii) optimistically pick an arm to play within the current neighborhood (\Crefrange{alg:line:forceLeaderToBePlayed}{MAUB:line:endComputeArmToPlay}),
and (iii) observe the results and update statistics (\Crefrange{alg:line:startObservedAndUpdate}{alg:line:endObservedAndUpdate}).

\paragraph{Leader Computations}
Leader computations aim at identifying the best arm, according to current mean statistics. They boil down to a call to greedy. Whenever a new leader is computed, neighborhood shall also be updated, meaning that Algorithm~\ref{alg:getNeighbors} is called.

\begin{algorithm}
\caption{Neighborhood Computation}
\label{alg:getNeighbors}
\begin{algorithmic}[1]
\Require base $B$, element values $\vmu$
\State $\mathrm{Neighbors} \gets \emptyset$
\State $\mathrm{NotMapped} \gets E\setminus B$
\For{$x \in B$ in order $\mu_{x_1}\leqslant\dots\leqslant\mu_{x_{\rg}}$}
    \For{$e \in \mathrm{NotMapped}$}
	\If{$B-x+e\in\cB$}
	\State $\mathrm{Neighbors} \gets \mathrm{Neighbors}\cup \{B-x+e\}$
	\State $\mathrm{NotMapped} \gets \mathrm{NotMapped}-e$
	\EndIf
    \EndFor
\EndFor
\State \textbf{Return} $\mathrm{Neighbors}$
\end{algorithmic}
\end{algorithm}

\begin{algorithm}
    \caption{\texttt{MAUB}}\label{alg:MAUB}
\begin{algorithmic}[1]
\State \text{Play every element once}
\State \text{Denote $\vhmu(0)$ current estimates of $\vmu$}
\State $L \gets \mathrm{greedy}(\vhmu(0))$
\State $\mathcal{N}\gets \mathrm{ComputeNeighbors}(L,\vhmu(0))$ 
\For{$t =  1, 2, \dots$}
\State \Comment{Update leader and neighborhood}
\If{$\hat{\mu}_L(t-1) < \max_{B \in \mathcal{N}} \hat{\mu}_B(t-1)$}
\label{alg:line:testIfNewLeaderNeeded}
\State $L \gets \mathrm{greedy}(\vhmu(t-1))$
\label{alg:line:greedy}
\State $\mathcal{N}\gets \mathrm{ComputeNeighbors}(L,\vhmu(t-1))$ 
\ElsIf{for the elements of $L$, the order $\hmu_{e_1}{\leqslant}\dots{\leqslant}\hmu_{e_D}$ has changed}
\label{MAUB:line:testIfNewNeighborsNeeded}
\State $\mathcal{N}\gets \mathrm{ComputeNeighbors}(L,\vhmu(t-1))$ 
\label{MAUB:line:computeNeighbors}
\EndIf
\label{MAUB:line:endUpdateLeaderAndNeighbors}
\State \Comment{Choose the arm to play}
\If{$l_L(t)-1 \equiv 0 [|E|-D+1]$} 
\label{alg:line:forceLeaderToBePlayed}
    \State $P \gets L$ 
\Else
\State $P \gets \argmax_{B \in \{L\} \cup \mathcal{N}} \vbmu_{B}(t, L)$
\label{MAUB:line:argmax}
\EndIf
\label{MAUB:line:endComputeArmToPlay}
\State \Comment{Play and update statistics}
\State $\mathrm{Observe}$ $X_e$ for all $e\in P$
\label{alg:line:startObservedAndUpdate}
\State $\forall e\in P,\ N_e (t) \gets N_e(t-1)+1$
\State $\forall e\in P,\ \hat{\mu}_e(t) \gets \frac{N_e(t-1)\cdot \hat{\mu}_e(t-1) + X_e}{N_e(t)}$
\State $L(t)\gets L$, $\cN(t)\gets \cN$, $P(t)\gets P$
\label{alg:line:endObservedAndUpdate}
\EndFor
\end{algorithmic}
\end{algorithm}

\paragraph{Picking the Arm to Play}
Identifying the arm to play within the current neighborhood of size at most $\gamma$ (Line~\ref{MAUB:line:argmax}) can be implemented in linear time $\mathcal{O}(\gamma)$ by remarking that, $\forall t, \forall B=\cL(t)-x+y \in \mathcal{N}_{\cL(t), \boldsymbol{\hat \mu}(t)}$, 
\begin{align}
    \sum_{e \in \cL(t)} \bmu_e(t) - \sum_{e\in B} \bmu_e(t) = \bmu_{x}(t) - \bmu_y(t).
\end{align}

\paragraph{Leader Changes and Neighborhood Updates}
A significant difference with previous unimodal bandits (\citep{com-icml-14, gau-icml-21,gau-icml-22}) is that \texttt{MAUB} does not recompute a leader at each iteration. Instead, and based on the regret analysis, it is sufficient to stick with the same leader $L(t)$ as in previous iterations, as long as according to mean statistics $\hat\vmu$ it has highest value within its neighborhood (Line~\ref{alg:line:testIfNewLeaderNeeded}).
Additionally, as long as the leader remains the same, the neighborhood stays valid provided the order $\hmu_{e_1} \leqslant \dots \leqslant\hmu_{e_{|L|}}$ of the leader’s elements is unchanged (Line~\ref{MAUB:line:testIfNewNeighborsNeeded}).
Verifying the validity of the leader and neighborhood is crucial, as it respectively saves $\cO(|E|)$ and $\cO(D(|E|-D))$ oracle calls whenever they remain unchanged.

\section{Theoretical Analysis}
\label{sec:MAUBAnalysis}

Let now analyze both the regret and the time complexity of \texttt{MAUB}. 

\subsection{Regret Upper-bound}

We start by restating the concentration inequality introduced by \citet{com-icml-14}, that is at the core of the analysis of \texttt{MAUB}. It allows bounding the expected deviation of an arm at specific time steps.

\begin{lemma}{\cite[Lemma B.1]{com-icml-14}}
    \label{lemma:CombesProutiere:B1}
    Let $k \in E$ and $\epsilon>0$. Define $\mathcal{F}_n$ the $\sigma$-algebra generated by $(X_k(t))_{t\leq n, k \in E}$. Let $\Lambda \subseteq \mathbb{N}$ be a random set of instants. Assume that there exists a sequence of random sets $(\Lambda(s))_{s\geq 1}$ such that (i) $\Lambda \subseteq \cup_{s\geq 1} \Lambda(s)$, (ii) $\forall s \geqslant 1,\ \forall n\in\Lambda(s)$, $t_k(n)\geq \epsilon s$, (iii) $\forall s,\ |\Lambda(s)|\leq 1$, and (iv) the event $\{n\in\Lambda(s)\}$ is $\mathcal{F}_n$-measurable. Then, $\forall \delta>0$:
    \begin{align}
	\mathds{E} \left[\sum_{n\geq 1} \mathds{1}\{n\in\Lambda, |\hat{\mu}_k(n) - \mu_k|>\delta\}\right] \leq \frac{1}{\epsilon \delta^2}
    \end{align}
\end{lemma}

The following lemma builds upon this result to bound the number of time steps for which unimodality is not satisfied. In light of Theorem~\ref{th:unimodality}, this can only happen whenever the order according to mean statistics within the elements of current leader $L(t)$ is wrong.

\begin{restatable}[]{lemma}{nbItBadNeighborhood}
    \ifdefined\Appendix
    (Number of iterations with incorrect neighborhood, originally stated on page~\pageref{lemma:incorrectNeighborhood})
    \else
    (Number of iterations with incorrect neighborhood, proof in App.~\ref{app:proof:nbItBadNeighborhood})
    \label{lemma:incorrectNeighborhood}
    \fi
    Let $B\in\cB$. Let $x^*=\argmin_{\tilde x\in B \setminus B^*}[\mu_{\tilde x}]$. Let $y^*\in \{\tilde y\in B^*\setminus B,\ B-x^*+\tilde y\in\cB\}$. 
    It holds that
    $$\mathds{E}[\sum_{s=1}^T \mathds{1}_{\{\cL(s)=B,\ B-x^*+y^* \notin \cN_{B,\hmu(s)}\}}]=\cO_{T\to\infty}(1).$$
\end{restatable}

Lemma~\ref{lemma:incorrectNeighborhood} permits following a similar proof architecture as \citet{com-icml-14}, since we may now stick with iterations satisfying unimodality, and bound the average time spent in suboptimal arms.
\begin{restatable}[]{theorem}{suboptimalArms}
    \ifdefined\Appendix
    (Time spent in suboptimal arms, originally stated on page~\pageref{th:suboptimalArms})
    \else
    (Time spent in suboptimal arms, Proof in App.~\ref{app:proof:suboptimalArms})
    \label{th:suboptimalArms}
    \fi
    Any suboptimal arm is the leader $\cO(\log\log(T))$ times in expectation.
\end{restatable}

The regret coming from suboptimal arms being leaders is thus negligible compared to the regret when $B^*$ is the leader, which resembles a regular bandit problem on the neighborhood of $B^*$. The final regret analysis of \texttt{MAUB} follows from this observation and is formalized in the following theorem.

\begin{restatable}[]{theorem}{regretMAUB}
    \ifdefined\Appendix
    Originally stated on page~\pageref{th:regretMAUB}
    \else
    (Proof in App.~\ref{app:proof:regretMAUB})
    \label{th:regretMAUB}
    \fi
    The regret of \ouralgo{} up to time step $T$ is $$\cO\left(\sum_{e \notin B^*}\frac{8}{\mu_{\sigma(e)} - \mu_{e}} \log(T)\right)
    =\cO\left(\frac{|E|-D}{\Delta_{\mathrm{min}}}\log T\right),$$
    with $\Delta_{\mathrm{min}} \eqdef\min_{e \notin B^*}\mu_{\sigma(e)} - \mu_{e}.$
\end{restatable}
The constant $\sum_{e \notin B^*}\frac{8}{\mu_{\sigma(e)} - \mu_{e}}$ is slightly better than \texttt{OMM}'s constant $\sum_{e \notin B^*} \frac{16}{\min_{i \in B^*, \mu_i>\mu_e}\mu_i - \mu_e}$ since for any element $e\notin B^*$, $\sigma(e)\in B^*$. Yet, we believe that a finer analysis of \texttt{OMM}'s behavior would show that \texttt{OMM} actually achieves the same regret as \texttt{MAUB}.

\subsection{Time Complexity}

The time complexity analysis of \texttt{MAUB} mainly builds upon Theorem~\ref{th:suboptimalArms}: suboptimal arms can only be leaders for a negligible amount of time steps. Hence $B^*$ is the leader for most time steps, and there are few iterations involving a leader change or a neighborhood update, which drastically decreases the overall number of oracle calls compared to state of the art. 

\begin{corollary}
    \label{cor:timeComplexity}
    The expected overall time complexity of \ouralgo{} is 
    \begin{align*}
    \nonumber
    &\cO(|E|\cdot T +
    |\cB|(|E|-D)^2 \log\log T \\
    & \qquad \cdot [|E|(\log |E|+\cT_m)+  D(|E|-D)\cT_m])\\
    &=\cO(|E|\cdot T + |\cB| \mathrm{poly}(|E|,D)\cT_m \log \log (T)),
    \end{align*}
    where $|\cB|\leqslant\binom{|E|}{\rg}$.
\end{corollary}

    The first term comes from identifying the arm to play in the neighborhood and the second one corresponds to neighborhood computations.
    In fact, all iterations take at least $\mathcal{O}(|E|-\rg)$, and there are $\mathcal{O}(\log\log(T))$ iterations such that the per-round time complexity suffers an additional time, which is detailed in the proof. 

\begin{proof}
    First, the identification of the arm to play requires at each iteration: the computation of $|E|$ optimistic terms, and $|E|-D$ comparisons of two optimistic terms. This induces a total computation complexity of $\cO(|E|\cdot T)$. Other tests and updates of statistics have the same computational complexity.

    Second, let us account for the changes of leader.
    Let $T\in\mathbb{N}^*\setminus\{1\}$ and let $C^T=\{n\leq T~|~\cL(n-1)\neq\cL(n)\}$. Because of the assumption that all bases have distinct values, $C^T\subseteq D^T\cup E^T$, where:
    \begin{itemize}
	\item $D^T=\{n\leq T~|~\mu_{\cL(n-1)}>\mu_{\cL(n)}\}$ is the number of suboptimal leader changes;
	\item $E^T=\{n\leq T~|~\mu_{\cL(n-1)}<\mu_{\cL(n)}\}$ is the number of leader changes improving current value.
    \end{itemize}

    The proof relies on Theorem~\ref{th:suboptimalArms} and remarking that in both $D^T$ and $E^T$, the leader changes either come from a suboptimal arm or lead to a suboptimal arm. 
    It holds that $D^T\subseteq\{n\leq T~|~\cL(n)\neq B^*\}$, and $E^T\subseteq\{n\leq T~|~\cL(n-1)\neq B^*\}$.
    Hence, from Theorem~\ref{th:suboptimalArms}, the expectation of the size of both $D^T$ and $E^T$ is $\cO(\gamma^2|\cB|\log\log(T))$, and then, there are $\cO(\gamma^2|\cB| \log\log T)$ leader changes, each requiring $\cO(|E|(\log |E| + \cT_m) +  D(|E|-D)\cT_m)$ time.

Finally, let us upper-bounds the expected computation cost of  neighborhood updates without leader change (Line~\ref{MAUB:line:computeNeighbors}).
The proof follows the same structure as the one on leader changes, except that we use Lemma~\ref{lemma:incorrectNeighborhood}.
It leads to a $\cO(D(|E|-D)\cT_m)$ cost for $T$ iterations, which is negligible with respect to other terms.
\end{proof}

\section{Experiments}
\label{sec:experiments}

In this section, \ouralgo{} is evaluated against \texttt{OMM} on four different matroid benchmarks, namely uniform, linear, graphic, and transversal matroids.
Section~\ref{sec:res} discusses the experimental results, 
but we first detail the experimental protocol in Section~\ref{sec:protocol}, and give usefull background on benchmarked matroids in Section~\ref{sec:domains}.

\subsection{Experimental protocol}\label{sec:protocol}

For each matroid, the learning algorithm observes, at time step $t$, realizations of Gaussian distributions $(X_e)_{e \in \cP(t)}$, where $\forall e,\ X_e \sim \mathcal{N}(\mu_e, \sigma^2)$. Mean values $\vmu$ are randomly set in $[0.5,1]$ (except for linear matroid, see corresponding paragraph) at the beginning of the experiment and the standard deviation $\sigma$ is set to $0.2$ for all distributions. 

Each experiment ran with 20 different seeds, and we capture (i) the regret (see Equation~\eqref{def:regret}), as a function of iterations, given in Figure~\ref{fig:regrets}; and (ii) overall time complexity, number of oracle and greedy calls, and \texttt{MAUB}'s number of order change within current leaders (Line~\ref{MAUB:line:computeNeighbors}) provided in Table~\ref{tab:timeComplexities}.

\paragraph{Reproducibility}

The code to reproduce the experiments is available in supplementary material.
All algorithms are implemented using Python3.11. We use the {\em Sagemath} library \citep{sagemath} to handle matroids. 
Graphic and uniform matroids are taken from Sagemath's database, while the linear and transversal matroids are constructed as Sagemath matroids using external data. More detail is given in their corresponding paragraphs. 
The experiments were run using one core of a 2.20 \texttt{GHz} Intel® Xeon® Gold 5320 \texttt{CPU}, part of a cluster with 384 \texttt{GiB} available \texttt{RAM}. 

\begin{table}
    \caption{Overall time complexity of oracle calls for different types of matroids.}
    \label{tab:oracleTime}
    \centering
\resizebox{\linewidth}{!}{
    \begin{tabular}{llll}
    \toprule
    \textbf{Matroid} &\textbf{Time Comp.} & {\bf Oracle}\\
        \midrule
    Uniform & $\cO(1)$ & cardinal test\\
    Graphic & $\cO(\log(|E|))$ & cycle test \\
    Transversal&  $\cO(\rg|E|)$ & augmenting path  on matching\\
    Linear&  $\cO(D^2)$ & Gaussian pivot step\\
\bottomrule
\end{tabular}
}
\end{table}

\begin{figure*}[ht]
    \centering
     \begin{subfigure}{0.23\textwidth}
    \centering
\resizebox{\linewidth}{!}{
\includegraphics{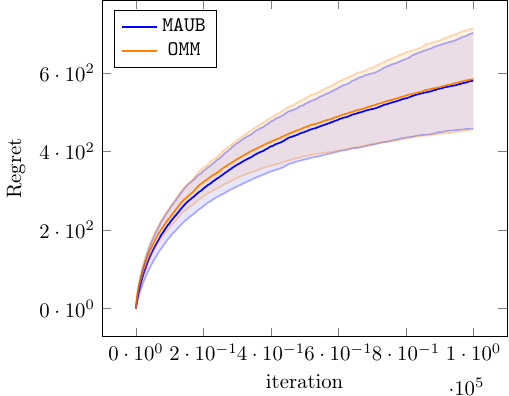}
}
\caption{Unif. matroid U(7,10)}
 \end{subfigure}
 \hfill
     \begin{subfigure}{0.23\textwidth}
    \centering
\resizebox{\linewidth}{!}{
\includegraphics{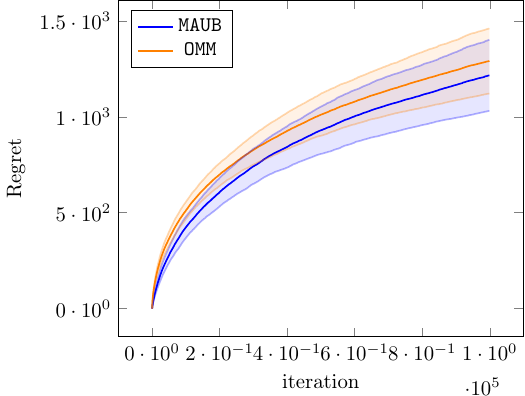}
}
\caption{Unif. matroid U(7,15)}
 \end{subfigure}
 \hfill
\begin{subfigure}{0.23\textwidth}
    \centering
\resizebox{\linewidth}{!}{
\includegraphics{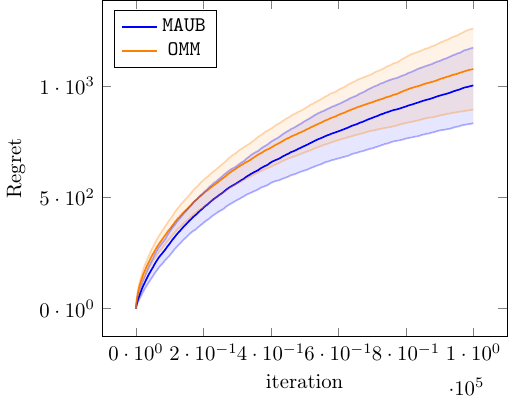}
}
\caption{Unif. matroid U(15,20)}
 \end{subfigure}
\hfill
     \begin{subfigure}{0.23\textwidth}
    \centering
\resizebox{\linewidth}{!}{
\includegraphics{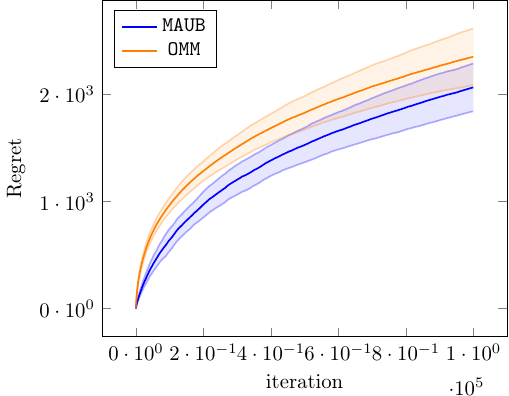}
}
\caption{Unif. matroid U(15,30)}
 \end{subfigure}

\begin{subfigure}{0.23\linewidth}
    \centering
\resizebox{\linewidth}{!}{
\includegraphics{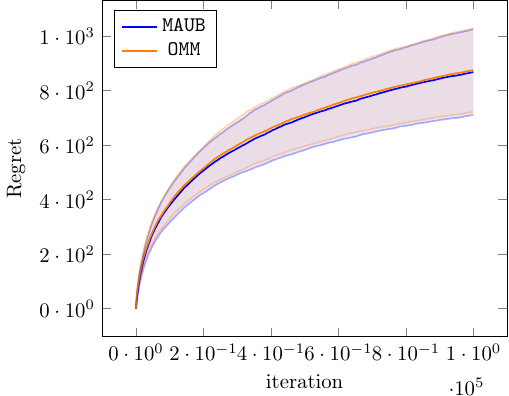}
}
\caption{Graphic matroid $K_5$}
\end{subfigure}
\hfill
\begin{subfigure}{0.23\linewidth}
    \centering
\resizebox{\linewidth}{!}{
\includegraphics{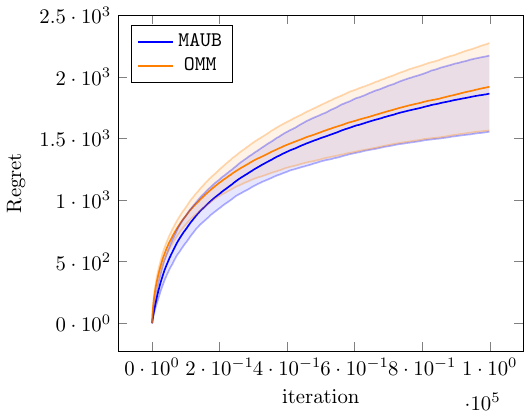}
}
\caption{Graphic matroid $K_7$}
\end{subfigure}
\hfill
\begin{subfigure}{0.23\linewidth}
    \centering
\resizebox{\linewidth}{!}{
\includegraphics{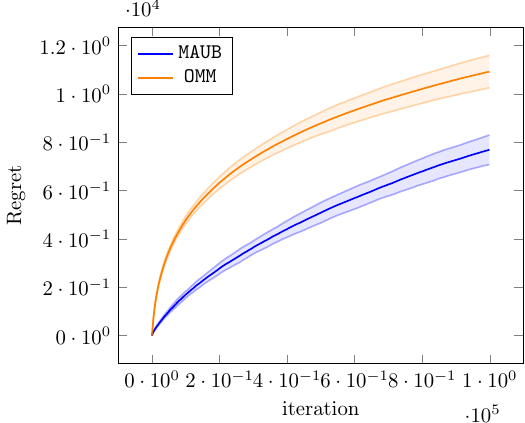}
}
\caption{Graphic matroid $K_{15}$}
\end{subfigure}
\hfill
\begin{subfigure}{0.23\linewidth}
    \centering
\resizebox{\linewidth}{!}{
\includegraphics{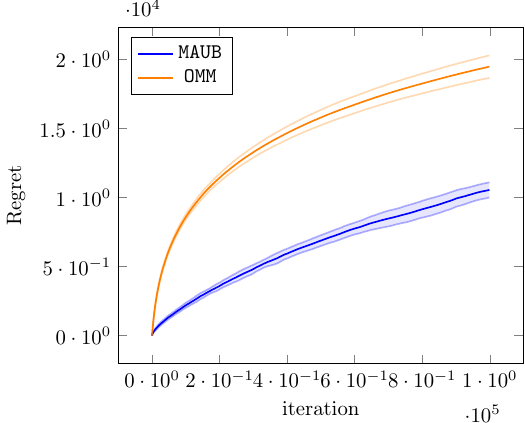}
}
\caption{Graphic matroid $K_{20}$}

\end{subfigure}
 \begin{subfigure}{0.23\linewidth}
    \centering
\resizebox{\linewidth}{!}{
\includegraphics{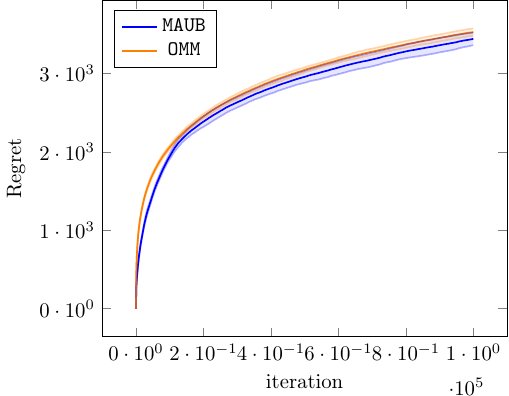}
}
\caption{Linear matroid}
\end{subfigure}
%
\hspace{1em}
\begin{subfigure}{0.23\linewidth}
    \centering
\resizebox{\linewidth}{!}{
\includegraphics{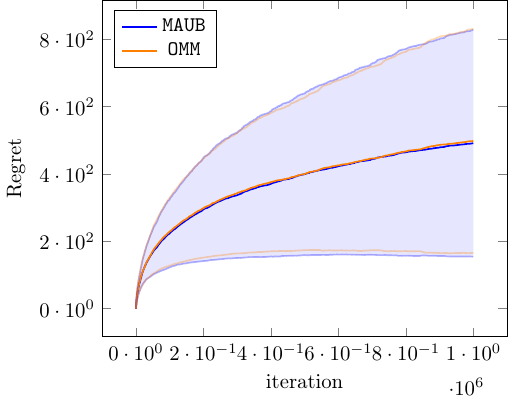}
}
\caption{Transversal matroid}
\end{subfigure}
\hfill
\caption{Regret vs iterations for uniform, linear, graphic and transversal matroids (the smaller, the better).
U($\rg$,$|E|$) is the uniform matroid of rank $\rg$ on $|E|$ elements.
$K_N$ it the graphic matroid associated to the complete graph of size $N$. 
Our algorithm \textbf{\ouralgo{}} consistently \textbf{matches or outperforms \texttt{OMM}}.}
\label{fig:regrets}
\end{figure*}

\begin{table*}
\caption{Time statistics for \ouralgo{} and \texttt{OMM} on several matroids.
Computation time is in seconds and other metrics correspond to numbers of calls (averaged over 20 runs). \texttt{Neigh. Up.} denotes the number of neighborhood updates performed by \ouralgo{} while the leader remains stable.
\ouralgo{} significantly \textbf{reduces the number of oracle calls} compared to \texttt{OMM}, which translates into \textbf{lower overall computation time}.
}
\label{tab:timeComplexities}
      \centering
\begin{subtable}{.49\linewidth}
    \label{tab:timeComplexityUniformMatroid}
    \caption{Uniform matroids}
    \centering
\resizebox{\linewidth}{!}{
    \begin{tabular}{lrrrc}
	\toprule
	\textbf{Algorithm} &\textbf{Time}& \textbf{Oracle Calls} & \textbf{Greedy Calls}& \textbf{Neigh. Up.}\\
	\midrule
\multicolumn{5}{c}{Uniform ($D=7$, $|E|=10$, $|\cB|=120$)}\\ 
\midrule
	\texttt{OMM} & 3.45\ s & \convert{700000.0} & \convert{100000.0}& -\\
	\texttt{MAUB} & {\bf 1.92\ s} & {\bf \convertBold{497.25}} & {\bf \convertBold{13.8}}&\convertBold{121.75}\\
\midrule
\midrule
\multicolumn{5}{c}{Uniform ($D=7$, $|E|=15$, $|\cB|=6435$)}\\ 
\midrule
	\texttt{OMM} & 4.44\ s & \convert{700000.0} & \convert{100000.0}& -\\
	\texttt{MAUB} & {\bf 3.84\ s} & {\bf \convertBold{2675.4}} & {\bf \convertBold{48.2}} &\convertBold{247.05}\\
\midrule
\midrule
\multicolumn{5}{c}{Uniform ($D=15$, $|E|=20$, $|\cB|=15504$)}\\ 
\midrule
 \texttt{OMM} & 6.87\ s & \convert{1500000.0} & \convert{100000.0}& -\\
 \texttt{MAUB} & {\bf 3.50\ s} & {\bf \convertBold{3327.75}} & {\bf \convertBold{40.2}}&\convertBold{506.75}\\
\midrule
\midrule
\multicolumn{5}{c}{Uniform ($D=15$, $|E|=30$, $|\cB|=\convert{155117520}$)}\\ 
\midrule
 \texttt{OMM} & 8.84\ s & \convert{1500000.0} & \convert{100000.0}& -\\
 \texttt{MAUB} & {\bf 7.51\ s} & {\bf \convertBold{12833.25}} & {\bf \convertBold{76.1}} &\convertBold{705.35}\\
\bottomrule
\end{tabular}
}
\end{subtable}
\hspace{\fill}
\begin{subtable}{.49\linewidth}
    \caption{Graphic matroids}
    \label{tab:timeComplexityGraphicMatroid}
    \centering
\resizebox{\linewidth}{!}{
    \begin{tabular}{lrrrc}
	\toprule
	\textbf{Algorithm} &\textbf{Time}& \textbf{Oracle Calls} & \textbf{Greedy Calls}& \textbf{Neigh. Up.}\\
	\midrule
\multicolumn{5}{c}{$K_{5}$ ($D=4$, $|E|=10$, $|\cB|=125$)}\\ 
 \midrule 
 \texttt{OMM} & 11.28\ s & \convert{656480.7} & \convert{100000.0}& -\\
 \texttt{MAUB} & {\bf 7.83\ s} & \convertBold{797.35} & \convertBold{23.05} & \convertBold{60.4} \\
\midrule
\midrule
\multicolumn{5}{c}{$K_{7}$ ($D=6$, $|E|=21$, $|\cB|=16807$)}\\ 
 \midrule 
 \texttt{OMM} & 22.24\ s & \convert{1443384.05} & \convert{100000.0}& -\\
 \texttt{MAUB} & {\bf 18.02\ s} & {\bf \convertBold{7030.15}} & {\bf \convertBold{64.7}} & \convertBold{166.35} \\
\midrule
\midrule
 \multicolumn{5}{c}{$K_{15}$ ($D=14$, $|E|=105$, $1.95 \cdot 10^{15}$)}\\ 
 \midrule 
 \texttt{OMM} & 131.89\ s & \convert{8836943.45} & \convert{100000.0}& -\\
 \texttt{MAUB} & {\bf 108.43\ s} & {\bf \convertBold{581680.95}} & {\bf \convertBold{184.9}} & \convertBold{1575.05} \\
\midrule
\midrule
\multicolumn{5}{c}{$K_{20}$ ($D=19$, $|E|=190$, $2.22 \cdot 10^{23}$)}\\ 
 \midrule 
 \texttt{OMM} & 262.55\ s & \convert{16687525.75} & \convert{100000.0}& -\\
 \texttt{MAUB} & {\bf 235.67\ s} & {\bf \convertBold{3501989.1}} & {\bf \convertBold{357.3}} & \convertBold{3765.8} \\
\bottomrule
\end{tabular}
}
\end{subtable}
\hspace{\fill}
\bigbreak
\begin{subtable}{.49\linewidth}
    \label{tab:timeComplexityLinearMatroid}
    \caption{Linear matroid}
    \centering
\resizebox{\linewidth}{!}{
    \begin{tabular}{lrrrc}
	\toprule
	\textbf{Algorithm} &\textbf{Time}& \textbf{Oracle Calls} & \textbf{Greedy Calls} & \textbf{Neigh. Up.}\\
	\midrule
	\multicolumn{5}{c}{Linear ($D=16$, $|E|=100$, $|\cB| \texttt{ unknown}$)}\\ 
	\midrule
	\texttt{OMM} &  72.69\ s & \convert{9698629.3} &  \convert{100000.0} & -\\
	\texttt{MAUB} & {\bf 60.49\ s} & {\bf \convertBold{227631.05}} & {\bf \convertBold{49.6}} &\convertBold{203.2}\\
	\bottomrule
    \end{tabular}
}
\end{subtable}
\hspace{\fill}
\begin{subtable}{.49\linewidth}
    \label{tab:timeComplexityTransversalMatroid}
    \caption{Transversal matroid}
    \centering
\resizebox{\linewidth}{!}{
    \begin{tabular}{lrrrc}
	\toprule
	\textbf{Algorithm} &\textbf{Time}& \textbf{Oracle Calls} & \textbf{Greedy Calls} & \textbf{Neigh. Up.}\\
	\midrule
	\multicolumn{5}{c}{Transversal ($D=6$, $|E|=7$, $|\cB|=7$)}\\ 
	\midrule
	\texttt{OMM} & 83.42\ s & \convert{6000000.0} & \convert{1000000.0} & -\\
	\texttt{MAUB} & {\bf 17.44\ s} & {\bf \convertBold{227.95}} & {\bf \convertBold{25.6}} &\convertBold{50.75}\\
	 \bottomrule
    \end{tabular}
}
\end{subtable}
\end{table*}

\subsection{Experimental Domains}\label{sec:domains}

We detail here the  matroids considered in the experiments.
The time complexity $\cT_m$ of the membership oracle for the four classes of matroid are given in Table~\ref{tab:oracleTime}.

\paragraph{Uniform Matroids}
\label{sec:exp:uniform}

In these experiments, all subsets of size at most $\rg$ are independent, so the problem reduces to identifying the  $\rg$ elements with the highest values.

For such matroids, the independence oracle simply checks the size of a subset $I$, which is done in $\cO(1)$ time.

\paragraph{Linear Matroid}
\label{sec:exp:linear}

Learning in linear matroids is illustrated here by searching for set of independent movies with highest rating. Similarly to \citeauthor{kve-uai-14} (\citeyear{kve-uai-14}) 100 movies from the dataset {\em Movielens} are considered.
The database attributes types, among 18 possible ones, to each movie.
A set of movies $I$ is independent if and only if the characteristic vectors $\{u_e\in\{0,1\}^{18}, e\in I\}$ form a linearly independent family of $\mathbb{R}^{18}$. The latter constraint ensures diversity to some extent. Values of movies are given by users' average ratings $\vmu$, with $\forall e,\ \mu_e \in [0,5]$.

The independence oracle for linear matroids involves one step in Gaussian pivot, which takes $\cO(D^2)$ time.

\paragraph{Graphic Matroids}
\label{sec:exp:graphic}

For any graph $G=(S,A)$, a graphic matroid is defined by the set of spanning trees of the graph.
µ
As \citeauthor{tal-aamas-16} (\citeyear{tal-aamas-16}), we consider the graphic matroid $K_N$ associated to the complete graph of size $N$. 

Checking whether adding an element $x$ to an independent set $I$ creates a circuit can be done in $\cO(\log(|E|)$ time, using union-find data structure.

\paragraph{Transversal Matroid}
\label{sec:exp:transversal}

Given a bipartite graph $G=(X \cup Y, A)$, one can define a matroid where the groundset is the vertices in $X$, and a subset $B \subseteq X$ is independent if and only if it admits a matching with $Y$. An arbitrary graph, available in supplementary material, with $|X|=7$, $|Y|=6$, and $17$ edges is considered. 

Assuming that $I$ is independent, and that a matching $M$ is stored, deciding whether $I+x$ is independent can be reduced to checking if an augmenting path in $G$ with respect to $M$ can be found. The independence oracle for transversal matroid thus takes $\cO(D|E|)$ time.

\subsection{Results}\label{sec:res}

Figure \ref{fig:regrets} shows no loss in regret of \texttt{MAUB} compared to \texttt{OMM} in asymptotic regime, and an improved regret in early iterations for large matroids ({\em e.g.} $U(15;30)$, linear matroid, $K_{15}$, $K_{20}$). We believe that this phenomenon is due \texttt{MAUB}'s restricted exploration, while \texttt{OMM} might test several suboptimal elements in the same iteration. 

The number of oracle calls given in Table \ref{tab:timeComplexities} is drastically reduced compared to \texttt{OMM} (ranging from one order of magnitude ({\em e.g.} $K_{20}$, Linear) to three ({\em e.g.} $K_{3,3}$, transversal)). This is due to the number of iterations performed by \texttt{MAUB} requiring calls to the matroid structure ({\em i.e.} number of greedy calls plus number of order change requiring neighborhood computations) being significantly lower than the number of greedy calls performed by \texttt{OMM}. 
As expected, this smaller number of oracle calls translates into lower computation time.

A side observation is that constant $C_{\cM}$ in Theorem \ref{cor:timeComplexity} is largely overestimated, as pointed out by the relatively low number of greedy calls performed by \texttt{MAUB} in experiments on matroids with intractable total number of bases ({\em e.g.} $U(15,20)$, $K_{20}$, linear).

\section{Discussion and Perspectives}

We introduced \texttt{MAUB}, a unimodal bandit approach for learning in matroid combinatorial semi-bandit problems. 
While optimal regret is already achieved by the state of the art for this combinatorial structure, repeated queries to the underlying optimization problem can be a major computational bottleneck for certain classes of matroids. 
As supported by the theoretical analysis and the empirical experiments on various classes of matroids, \ouralgo{} leverages the unimodality of matroids to drastically reduce the number of oracle calls which results in an improved overall time complexity (for no loss in terms of regret).
Beyond our study, we believe that the present contribution lays down a promising path to tackle time-complexity reduction in other semi-bandit problems.

Our work also underscores that the current unimodal bandit analysis lacks a fine-grained characterization of actual early trajectories taken by the learning algorithm. 
For instance, current theoretical analyses of combinatorial unimodal bandits ignore correlation between suboptimal arms, leading to a large overestimation of some constant terms in the overall regret and time complexity.
Filling this theoretical gap would widen the relevance of the unimodal approach to the non-asymptotic regime, beyond the specific matroid structure considered in this paper.

\bibliography{biblio.bib}

\setcounter{section}{0}
\renewcommand{\thesection}{\Alph{section}}
\onecolumn

\def\Appendix

\section{Proof Diagram}

To help the reader, we provide below (\Cref{fig:proofDiagram}) a proof diagram summarizing the finite-time analysis of \texttt{MAUB} conducted in the appendix.

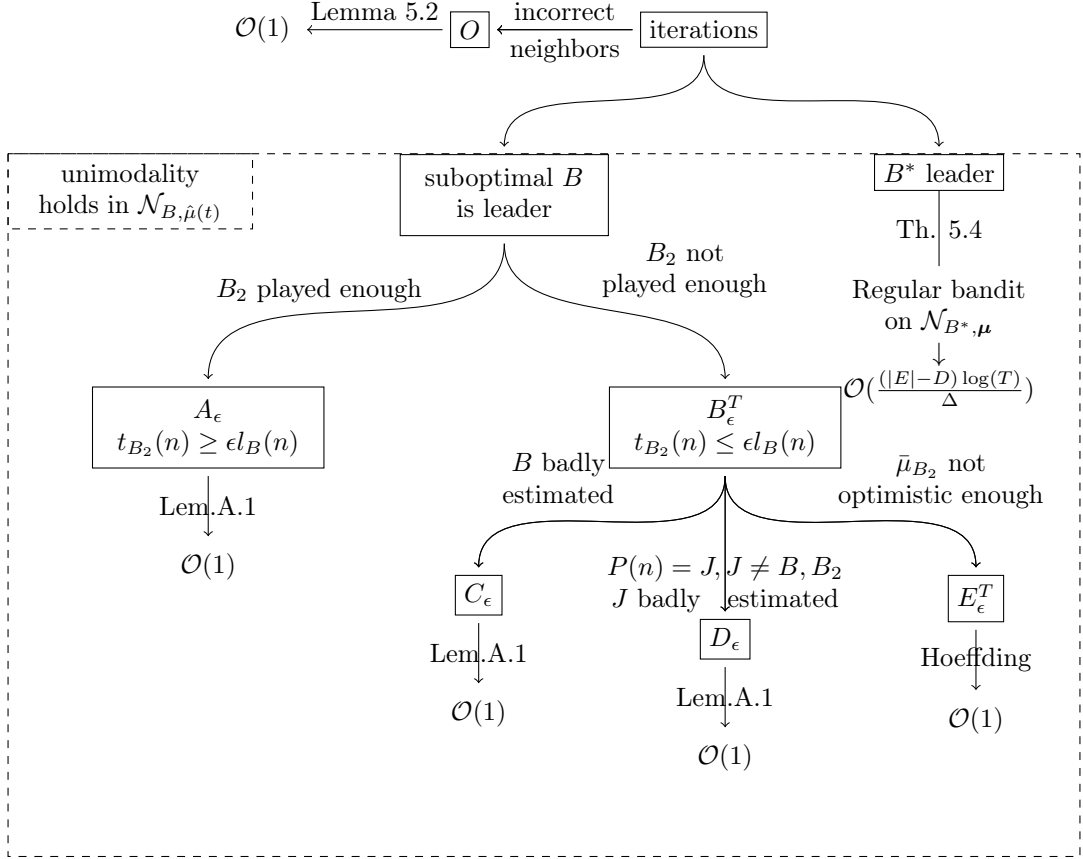
\begin{figure}[H]
\centering
\begin{tikzpicture}[block/.style={draw}]

\node[block] (1) {iterations};
\node[block, left=2cm of 1] (0) {$O$};
\node[left=2cm of 0] (tmp0) {$\cO(1)$};
\node[block, below left= 2 of 1, xshift=1cm] (2) {\begin{tabular}{c} suboptimal $B$\\ is leader \end{tabular}};
\node[block, below right= 2 of 1] (3) {$B^{*}$ leader};
\node[below= 1 of 3] (RegularBandit) {\begin{tabular}{c} Regular bandit \\ on $\cN_{B^*,\vmu}$ \\ $\downarrow$  \\ $\cO(\frac{(|E|-D)\log(T)}{\Delta})$ \end{tabular}};
\node[block, below left= 2 and 1 of 2] (Aeps) {\begin{tabular}{c} $A_\epsilon$ \\ $t_{B_2}(n) \geq \epsilon l_B(n)$ \end{tabular}};
\node[below=1 of Aeps] (AepsVal) {$\cO(1)$};
\node[block, below right= 2 and 1 of 2, xshift=-1cm] (Beps) {\begin{tabular}{c} $B_\epsilon^T$ \\ $t_{B_2}(n) \leq \epsilon l_B(n)$ \end{tabular}};
\node[block, below left=2 of Beps] (C) {$C_\epsilon$};
\node[below=1 of C] (Cval) {$\cO(1)$};
\node[block, below =2 of Beps] (D) {$D_\epsilon$};
\node[below=1 of D] (Dval) {$\cO(1)$};
\node[block, below right=2 of Beps] (E) {$E_\epsilon^T$};
\node[below=1 of E] (Eval) {$\cO(1)$};

\draw (3) to[out=-90,in=90] node[] {Th. \ref{th:regretMAUB}} (RegularBandit);
\begin{scope}[->, shorten >=1mm, shorten <=1mm]
\draw (1) to[out=-180,in=0] node[above]{incorrect} (0);
\draw (1) to[out=-180,in=0] node[below]{neighbors} (0);
\draw (0) to[out=-180,in=0] node[above] {Lemma \ref{lemma:incorrectNeighborhood}} (tmp0);
\draw (1) to[out=-90,in=90] (3);
\draw (1) to[out=-90,in=90] (2);
\draw (2) to[out=-90,in=90] node[above left=0 and -1] {$B_2$ played enough} (Aeps);
\draw (Aeps) to[out=-90,in=90] node[] {Lem.\ref{app:lemma:CombesProutiere:B1}} (AepsVal);
\draw (2) to[out=-90,in=90] node[above right=0 and -.5] {\begin{tabular}{c} $B_2$ not \\ played enough \end{tabular}} (Beps);
\draw (Beps) to[out=-90,in=90] (C);
\draw (C) to[out=-90,in=90] node[] {Lem.\ref{app:lemma:CombesProutiere:B1}} (Cval);
\draw (Beps) to[out=-90,in=90] (D);
\draw (D) to[out=-90,in=90] node[] {Lem.\ref{app:lemma:CombesProutiere:B1}}(Dval);
\draw (Beps) to[out=-90,in=90] (E);
\draw (E) to[out=-90,in=90] node[] {Hoeffding}(Eval);
\draw (Beps) to[out=-90,in=90] node[above left=0 and -.5] {\begin{tabular}{c} $B$ badly \\ estimated \end{tabular}} (C);
\draw (Beps) to[out=-90,in=90] node[below] {\begin{tabular}{c} $\cP(n)=J,  J \neq B,B_2$\\ $J$ badly ~ estimated \end{tabular}} (D);
\draw (Beps) to[out=-90,in=90] node[above right=0 and -.5] {\begin{tabular}{c} $\bmu_{B_2}$ not \\ optimistic enough\end{tabular}} (E);
\draw[dashed] (-9.25,-1.65) rectangle (5,-11);
\draw[dashed] (-9.25,-1.65) rectangle node {\begin{tabular}{c} unimodality\\ holds in $\cN_{B, \hmu(t)}$ \end{tabular}} (-6,-2.65);
\end{scope}
\end{tikzpicture}
\caption{Proof diagram of finite-time analysis of \texttt{MAUB}. We remind the reader that whenever unimodality is satisfied, $B_2$ is an arm in $\cN_{B, \vmu}$ such that $\vmu_{B_2}>\vmu_{B}$.}
\label{fig:proofDiagram}
\end{figure}

\section{Concentration Inequality}

We start by restating two core lemmas introduced by \citet{com-icml-14}. They allow bounding the expected number of high deviations of mean statistics at specific time steps.

\begin{lemma}{\cite[Lemma B.1]{com-icml-14}}
    \label{app:lemma:CombesProutiere:B1}
    Let $k \in E$ and $\epsilon>0$. Define $\mathcal{F}_n$ the $\sigma$-algebra generated by $(X_k(t))_{t\leq n, k \in E}$. Let $\Lambda \subseteq \mathbb{N}$ be a random set of instants. Assume that there exists a sequence of random sets $(\Lambda(s))_{s\geq 1}$ such that (i) $\Lambda \subseteq \cup_{s\geq 1} \Lambda(s)$, (ii) $\forall s \geqslant 1,\ \forall n\in\Lambda(s)$, $t_k(n)\geq \epsilon s$, (iii) $\forall s,\ |\Lambda(s)|\leq 1$, and (iv) the event $\{n\in\Lambda(s)\}$ is $\mathcal{F}_n$-measurable. Then, $\forall \delta>0$:
    \begin{align}
	\mathds{E} \left[\sum_{n\geq 1} \mathds{1}\{n\in\Lambda, |\hat{\mu}_k(n) - \mu_k|>\delta\}\right] \leq \frac{1}{\epsilon \delta^2}
    \end{align}
\end{lemma}

\begin{lemma}{\cite[Lemma B.2]{com-icml-14}}
    \label{lemma:CombesProutiere:B2}
    Let $k,k' \in E$, $k \neq k'$ and $\epsilon>0$. Define $\mathcal{F}_n$ the $\sigma$-algebra generated by $(X_k(t))_{t\leq n, k \in \{1,\dots,[E|\}}$. Let $\Lambda \subset \mathbb{N}$ be a random set of instants. Assume that there exists a sequence of random sets $(\Lambda(s))_{s\geq 1}$ such that (i) $\Lambda \subset \cup_{s\geq 1} \Lambda(s)$, (ii) $\forall s \geq 1,\ \forall n\in\Lambda(s)$, $t_k(n)\geq \epsilon s$, and $t_{k'}(n)\geq \epsilon s$, (iii) $\forall s,\ |\Lambda(s)|\leq 1$ almost surely, (iv) $\forall n\in \Lambda,\ \mathds{E}[\hat{\mu_k}(n)]\leq\mathds{E}[\hat{\mu}_{k'}(n)] - \Delta_{k,k'}$, and (v) the event $\{n\in\Lambda(s)\}$ is $\mathcal{F}_n$-measurable. Then, $\forall \delta>0$:
    \begin{align}
	\mathds{E} [\sum_{n\geq 1} \mathds{1}\{n\in\Lambda, \hat{\mu}_k(n) > \hat{\mu}_{k'}(n)\}] \leq \frac{8}{\epsilon \Delta_{k,k'}^2}
    \end{align}
\end{lemma}

\section{Incorrect Neighborhood}
\label{app:proof:nbItBadNeighborhood}

This section proves that the number of iterations for which unimodality is not ensured is finite, in expectation. The proof relies on remarking that Th.\ref{th:unimodality} is not guaranteed to hold only for wrong ordering $\hmu_{e_{i_1}}<\dots<\hmu_{e_{i_{|L(t)|}}}$ of items in the current leader $L(t)$. 

\nbItBadNeighborhood*

\begin{proof}
    Let us define 
    \begin{align*}
	O&\eqdef\{s\leq T~|~\cL (s)=B,\ B-x^*+y^*\notin \cN_{B,\hmu(s)}\},
    \end{align*}
    and denote $\delta=\min_{a,b \in E, a \neq b}|\mu_a-\mu_b|$ the minimal gap between two elements in $E$.
    By design of Algorithm~\ref{alg:getNeighbors},
    \begin{align*}
	O&= \{s\leq T~|~ B=\cL (s),\ \exists x\in \{\tilde x\in B,\ B-\tilde x+y^*\in\cB\}, \hmu_{x^*}(s)>\hmu_{x}(s)\} \\
	&\subseteq \cup_{x \in B-x^*} \{s\leq T\mid\cL(s)=B,\
    \hmu_{x^*}(s)>\hmu_x(s)\}\\
	&\subseteq \cup_{x \in B-x^*} \{s\leq T\mid\cL(s)=B,\ |\hmu_{x^*}(n)-\mu_{x^*}|>\frac{\delta}{2} \vee |\hmu_{x}(n)-\mu_{x}|>\frac{\delta}{2}\} \\
	&\subseteq \bigcup_{x \in B}  \{s\leq T\mid\cL(s)=B,\ |\hmu_{x}(n)-\mu_{x}|>\frac{\delta}{2}\}.
    \end{align*}
    
Let $x\in B$ and bound the expected size of $O_x\eqdef\{s\leq T\mid\cL(s)=B,\ |\hmu_{x}(n)-\mu_{x}|>\frac{\delta}{2}\}$. Let $\Lambda\eqdef\{n\leq T,\ \cL(n)=B\}$. 
It holds that $O_x=\{n\in\Lambda,\ |\hmu_{x}(n)-\mu_{x}|>\frac{\delta}{2}\}$. 
For all $s$, let us define $\Lambda(s)=\{n \leq T,\ \cL(n)=B,\ l_B(n)=s\}$ 
so that $\Lambda \subseteq \cup_s \Lambda(s)$.
Since $l_B$ is increased each time $B$ is the leader, $\forall s,\ |\Lambda(s)|\leq 1$.
    $x$ being in $B$, and by design of \texttt{MAUB} forcing the leader to be played enough, for any $s\geqslant1$ and $n\in\Lambda(s)$, $t_{x}(n) \geq t_B(n) \geq l_B(n)/(\gamma+1)=s/(\gamma+1)$. Then, by Lemma~\ref{app:lemma:CombesProutiere:B1}, for $\epsilon\eqdef1/(\gamma+1)$
    \begin{align*}
    \mathds{E}\left[|O_x|\right] = 
	\mathds{E}\left[\sum_n \mathds{1}_{\{n\in\Lambda,\ |\hmu_{x}(n)-\mu_{x}|>\frac{\delta}{2}\}}\right]\leq\frac{4}{\epsilon\delta^2}.
    \end{align*}
    Since this holds for any $x$, 
    $$\mathds{E}[|O|]
    \leqslant \sum_x \mathds{E}\left[|O_x|\right]
    \leqslant |B|\frac{4}{\epsilon\delta^2}
    =\cO(1).$$
\end{proof}

\section{Suboptimal Arms}
\label{app:proof:suboptimalArms}

We are now ready to prove the key theorem for the regret analysis, stating that suboptimal arms can only be leaders for a $\cO(\log\log T)$ amount of time. The rest of the time then must be spent with the optimal arm being the leader, which resembles a regular bandit.

\suboptimalArms*

\begin{proof}

Let $B$ be a suboptimal arm, and let $B^*=\argmax_{B\in\cB} \mu_B$ be the global optimum. 
Let $B_2 \eqdef B-x^*+y^* = \argmax_{J \in \cN_{B, \vmu}} \mu_J$ be the best arm in the neighborhood of $B$.

It holds that 
\begin{align}
\{n \leq T~|~\cL(n)=B\} \subseteq V \cup A_\epsilon \cup B_\epsilon^T,
\end{align}
where:
\begin{itemize}
    \item $V\eqdef\{n~|~\cL(n)=B,\ \exists x,y \in \cL(n), \hmu_x(n)>\hmu_y(n), \mu_x<\mu_y\}$; and 
    \item for any $n\notin V$, it holds that $B_2 \in \cN_{B,\vhmu(n)}$, and  
    $A_\epsilon$ and $B_\epsilon^T$ can be defined as follows: 
	\begin{itemize}
	    \item $A_\epsilon\eqdef\{n\modif{~\notin V}~|~\cL(n)=B, t_{B_2}(n) \geq \epsilon l_B(n)\}$;
	    \item $B_\epsilon^T\eqdef\{n\leq T, \modif{n\notin V}~|~\cL(n)=B, t_{B_2}(n) \leq \epsilon l_B(n)\}$.
	\end{itemize}
\end{itemize}

\begin{remark}
Note that $V$ resembles $O$, the number of iterations which do not satisfy unimodality. We artificially also include in $O$ the time steps with wrong ordering for at least two elements in $B$. The proof should hold using $O$, but $V$ allows for a {\em static} neighborhood when bounding $B_\epsilon^T$, which circumvents technical details, with no loss in asymptotic regret. 
\end{remark}

Similarly to bounding $|O|$, we show that $|V|=\cO(1)$.
We use the following decomposition:
    \begin{align}
        V &\subseteq \bigcup_{x,y \in B^2, \mu_x < \mu_y} \{n \in \Lambda_{x,y}~|~\hmu_x>\hmu_y\};
    \end{align}
    where $\forall x,y, \forall s, \Lambda_{x,y}(s) \eqdef \{n~|~\cL(n)=B, l_B(n)=s\}$. For all $x,y$, 
    \begin{itemize}
        \item $\forall s, |\Lambda_{x,y}(s)| \leq 1$;
        \item $\forall s, \forall n \in \Lambda(s),\ t_x(n) \geq t_B(n) \geq l_B(n)/(\gamma+1)=s/(\gamma+1)\geq \epsilon s$, for some $\epsilon<1/(\gamma+1)$, and similarly for $y$.
    \end{itemize}
    Then, it holds that:
    \begin{align}
        V   &\subseteq \bigcup_{x,y \in B^2, \mu_x < \mu_y} \{n \in \Lambda_{x,y}~|~\hmu_x>\mu_x+\frac{\delta}{2} \vee \hmu_y<\mu_y-\frac{\delta}{2}\} \\
            & \subseteq \bigcup_{x,y \in B^2, \mu_x < \mu_y} \{n \in \Lambda_{x,y}~|~\hmu_x>\mu_x+\frac{\delta}{2}\} \cup \{n\in\Lambda_{x,y}~|~\hmu_y<\mu_y-\frac{\delta}{2}\},
    \end{align}
    where $\delta=\min_{a,b \in E, a \neq b}|\mu_a - \mu_b|$.
    It follows from Lemma~\ref{app:lemma:CombesProutiere:B1} that $|V| \leq D^2 \frac{8}{\epsilon \delta^2} = \cO(1)$.

We below bound the number of iterations at which a suboptimal arm $B$ is the leader, given that $B_2\in\cN_{B,\vhmu(t)}$.

\paragraph{\underline{$\mathds{E}[A_\epsilon]<\infty$ for any $T$}}

Let $n \in A_\epsilon$, and let $s\eqdef l_B(n)$.
By design of Alg.~\ref{alg:MAUB}, $t_B(n) \geq s/(\gamma+1)$, and by definition of $A_\epsilon$, $t_{B_2}(n)\geq \epsilon l_B(n)=\epsilon s.$ 
In light of those observations, the following applies Lemma~\ref{lemma:CombesProutiere:B2}. 
For any $s$, let $\Lambda(s)=\{n\in A_\epsilon| l_B(n)=s\}$, and $\Lambda=\cup_{s\geq 1} \Lambda(s)$. 
\begin{itemize}
    \item As any $n$ in $A_\epsilon$ is required to verify $\cL(n)=B$, and since $l_B(\cdot)$ increases each time $B$ is leader, $|\Lambda(s)|\leq 1$ (the time step such that $B$ is leader for the $s$-th time is unique).
    \item At any $p\in A_\epsilon$, either $B$ was already leader before, either a leader change occurred and led to $B$. In case of a leader change, a greedy algorithm was performed (Line~\ref{alg:line:greedy}) and $\vhmu_{B}$ was therefore better than $\vhmu{B}_2$. In case of $B$ staying leader, $\vhmu_{B}$ was better than $\vhmu_{B_2}$ (Line~\ref{alg:line:testIfNewLeaderNeeded}). 
    It therefore holds that $p\in A_\epsilon \implies p\in\{k\in\Lambda, \vhmu_{B}(k)\geq\vhmu_{B_2}(k)\}$, meaning that $A_\epsilon \subset \{k\in\Lambda, \vhmu_B(k)\geq\vhmu_{B_2}(k)\}$. 
	Using the fact that $B_2=B-x^*+y^*$, we have $A_\epsilon\subseteq\{p\in\Lambda, \hmu_{x^*}(p) \geq \hmu_{y^*}(p)\}$.
    \item Let $n\in \Lambda(s)$, for some $s$. By definition, $l_B(n)=s$. \texttt{MAUB} forces the leader to be played at least $\lfloor 1/(\gamma+1) \rfloor \cdot s$ times, so that $t_B(n)\geq \lfloor 1/(\gamma+1)\rfloor \cdot s$. Picking\footnote{Later, another constraint shall be put on the choice of $\epsilon$, but both are coherent.} epsilon sufficiently small with respect to the fixed $\gamma$ constant, it holds that $t_B(n)\geq\epsilon\cdot s$. It follows that $t_{x^*}(n)\geq t_B(n)\geq \epsilon \cdot s$.
	\item Let $n\in \Lambda(s)\subset A_\epsilon$, for some $s$. By definition of $A_\epsilon$, $n$ is such that $t_{B_2}(n) \geq \epsilon l_B(n) = \epsilon \cdot s$. It follows that $t_{y^*}(n)\geq t_{B_2}(n) = \epsilon \cdot s$.
	\item Let $n\in\Lambda(s)$. By assumption on the bandit problem, $\mathds{E}[\vhmu_{B_2}]-\mathds{E}[\vhmu_{B}]=\mathds{E}[\hmu_{y^*}] - \mathds{E}[\hmu_{x^*}] = \mu_{y^*}-\mu_{x^*} \geq \delta$ (with $\delta=\min_{a,b\in E, a \neq b} |\mu_a-\mu_b|$). 
    \item 
	Finally, by design of \texttt{MAUB}, the event $\{n\in\Lambda(s)\}$ clearly is measurable w.r.t. $(X_e(t))_{e\in E,1\leq t\leq T}$.
\end{itemize}
We thus apply Lem.~\ref{app:lemma:CombesProutiere:B1}, and get $\mathds{E}[|A_\epsilon|]<\infty$, for any $\epsilon, T$.

\paragraph{\underline{$\mathds{E}[B_\epsilon^T]<\infty$ for any $T$}}

Let:
\begin{itemize}
    \item $C_\delta \eqdef \{n\notin V~|~\cL(n)=B, \exists e\in B,\ |\hat{\mu}_e(n) - \mu_e|>\delta\}$ be the set of instants at which $B$ is badly estimated on at least one element $e$;
    \item $D_\delta \eqdef \cup_{J\in\cN_{B,\vmu}\setminus\{B_2\}} D_{\delta, J}$, where $\forall J,\ D_{\delta,J}\eqdef\{n\notin V~|~\cL(n)=B, \cP(n)=J,\ \exists w\in J,\  |\hat{\mu}_w(n) - \mu_w|>\delta\}$ is the set of instants at which $B$ is the leader, $J$ is chosen to be played while being badly estimated on at least one element $w$;
    \item $E^T\eqdef\{n\leq T, n\notin V~|~\cL(n)=B, \exists e\in B_2,\ \bar{\mu}_e(n)\leq\mu_{e}\}$ be the set of instants at which $B$ is the leader and the optimistic statistic $\vbmu$ underestimates at least one value $\mu_e$ for $e$ in $B_2$. 
\end{itemize}

The first step to bound $B_\epsilon^T$ is to show that for a properly chosen $\epsilon$, $|B_\epsilon^T|=2\gamma(1+\gamma)\cdot[|C_\delta|+|D_\delta|+|E^T|]+\cO_{T\to\infty}(1)$. 
Let $n\in B_\epsilon^T$ be a time step, let $B$ be the leader at $n$, and let $B_2$ be the best arm in $\cN_{B, \vmu}$. It holds that 
\begin{align}
l_B(n)=t_{B,B_2}+\sum_{J\in\cN_{B,\vmu}\cup\{B\}} t_{B,J}(n).
\end{align}
In other words, the number of time that $B$ was leader equals the sum, for all $J$ in $\cN_{B,\vmu}$, of the number of times $J$ was chosen while $B$ was the leader. 
Some information are known for $B_2$ and $B$: $t_{B_2}(n)\leq\epsilon l_B(n)$ as $n\in B_\epsilon^T$, and $t_B(n) \geq l_B(n)/(\gamma+1)$.
It follows that
\begin{align}
    (1-\epsilon)l_B(n) \leq \sum_{J\in\cN_{B,\vmu} \cup \{B\} \setminus\{B_2\}} t_{B,J}(n).
\label{eq:leaderNumberAsArmPlayed}
\end{align}
Let $\epsilon<\frac{1}{2\cdot (\gamma+1)}$\footnote{Note that this choice is coherent with the previous constraint on $\epsilon$ ($\epsilon<1/(\gamma+1)$)}\footnote{The following aims at knowing either that $B$ was played enough for it to be unlikely that it was chosen over $B_2$ (while $B_2$ is not overestimated, else, $n\in E^T$ holds!) or another $J$ was played enough for it to be unlikely that $J$ was picked over $B_2$. $B$ and $J$ must be treated differently as $t_{B,B}(n)\geq s/(\gamma+1)$ by design of \texttt{MAUB}.}. We show that it is impossible that both:
\begin{itemize}
    \item $\forall J\in\cN_{B, \vmu}\setminus\{B_2\},\ t_{B,J}<l_B(n)/(\gamma+1)$; and
    \item $t_{B,B}<(3/2)l_B(n)/(\gamma+1)+1$
\end{itemize}
hold.
Assume that both hold. 
Then, we show that the right-hand side of Eq.\ref{eq:leaderNumberAsArmPlayed} can be strictly upper-bounded by $(1-\epsilon)l_B(n)$, which creates an absurdity.
\begin{align}
    \sum_{J\in\cN_{B,\vmu}\setminus\{B_2\}}t_{B,J}(n) &< (3/2)\frac{l_B(n)}{(\gamma+1)}+1 + \sum_{J\in\cN_{B,\vmu}\setminus\{B_2\}} t_{B,J}(n) \\
				  &< (3/2)\frac{l_B(n)}{(\gamma+1)}+1 + (\gamma-1) \cdot \frac{l_B(n)}{(\gamma+1)} \qquad \text{(since $\gamma=|\cN_{B,\vmu}|$)} \\
				      &\leq \frac{l_B(n)}{(\gamma+1)} \cdot [(3/2)+(\gamma-1)] \\
				      &=\frac{l_B(n)}{2(\gamma+1)}\cdot [2\cdot\gamma+1].
\end{align}
By assumption on $\epsilon$, it also holds that:
\begin{align}
    (1-\epsilon)l_B(n) &> (1-\frac{1}{2(\gamma+1)})l_B(n) \\
		       &= (\frac{2\gamma+1}{2(\gamma+1)})l_B(n).
\end{align}
We conclude that $\sum_{J\in\cN_{B,\vmu}\cup\{B\}\setminus\{B_2\}}t_{B,J}(n)<(1-\epsilon)l_B(n)$, which is absurd, given \Cref{eq:leaderNumberAsArmPlayed}. 

\paragraph{(a)}
Assume that $\exists J\in\cN_{B,\vmu}\setminus\{B_2\},\ t_{B,J}(n)\geq l_B(n)/(\gamma+1)$. Let $J$ be such neighbor. Let $\phi(n)$ be the unique time step such that $\cL(\phi(n))=B$, $\cP(\phi(n))=J$, and $t_{B,J}(\phi(n))=\lfloor l_B(n)/2\cdot (\gamma+1) \rfloor$. 
$\phi(n)<n$ exists as (i) $t_{B,J}$ is only incremented when $B$ is the leader and $J$ is selected; and (ii) $n \mapsto l_{B}(n)$ is increasing. The unicity of $\phi$ comes from $t_{B,J}(\cdot)$ being increased at $\phi(n)$. 
In the following, we assume $\phi (n) \notin C_\delta \cup D_\delta \cup E^T$. 
Let $(x_{B_2}^*, y_{B_2}^*, x_J, y_J) \in [B\times B_2] \times [B \times J]$ be uniquely defined by:
\begin{align}
& B_2=B-x_{B_2}^{*}+y_{B_2}^{*}\text{, and}\\
& J=B-x_{J}+y_{J}.
\end{align}
The comparison between $J$ and $B_2$ boils down to a study of those four elements at $\phi(n)$.
Note that there is no reason for $B_2$ to be in $\cN_{J,\vmu}$, and conversely. 
Yet, $\phi(n)$ not belonging to $C_\delta \cup D_\delta \cup E^T$ implies that (i) $x_J$, $y_J$, and $x_{B_2}$ are correctly estimated up to the constant $\delta$, and (ii) $\bmu_e>\mu_e$ for all $e$ in $B_2$.
By definition, 
\begin{align}
    \bmu_{J}(\phi (n)) = \sum_{e \in J} \hmu_e(\phi(n)) + \sum_{e\in J} \sqrt{\frac{2\cdot\log(l_B(\phi(n)))}{t_e(\phi(n))}},
\end{align}
so that 
\begin{align}
    &\bmu_{J}(\phi (n))- \bmu_{B_2}(\phi(n)) \\
    &= -\hmu_{x_J}(\phi(n)) + \hmu_{y_J}(\phi(n)) + \hmu_{x^*_{B_2}}(\phi(n)) - \hmu_{y^*_{B_2}}(\phi(n)) \\
    & - \sqrt{\frac{2\cdot\log(l_B(\phi(n)))}{t_{x_J}(\phi(n))}} + \sqrt{\frac{2\cdot\log(l_B(\phi(n)))}{t_{y_J}(\phi(n))}}
    + \sqrt{\frac{2\cdot\log(l_B(\phi(n)))}{t_{x^*_{B_2}}(\phi(n))}} - \sqrt{\frac{2\cdot\log(l_B(\phi(n)))}{t_{y^*_{B_2}}(\phi(n))}}\\
    &< - (\mu_{x_J}-\delta) +\mu_{y_J}+\delta+\mu_{x_{B_2}^*}+\delta - \hmu_{y_{B_2}^*}(\phi(n))\\
    & - \sqrt{\frac{2\cdot\log(l_B(\phi(n)))}{t_{x_J}(\phi(n))}} + \sqrt{\frac{2\cdot\log(l_B(\phi(n)))}{t_{y_J}(\phi(n))}}
    + \sqrt{\frac{2\cdot\log(l_B(\phi(n)))}{t_{x^*_{B_2}}(\phi(n))}} - \sqrt{\frac{2\cdot\log(l_B(\phi(n)))}{t_{y^*_{B_2}}(\phi(n))}}\\
    &=-\mu_{x_J}+\mu_{y_J}+3\delta + \mu_{x_{B_2}^*} - \hmu_{y_{B_2}^*}(\phi(n))\\
    &- \sqrt{\frac{2\cdot\log(l_B(\phi(n)))}{t_{x_J}(\phi(n))}} + \sqrt{\frac{2\cdot\log(l_B(\phi(n)))}{t_{y_J}(\phi(n))}}
    + \sqrt{\frac{2\cdot\log(l_B(\phi(n)))}{t_{x^*_{B_2}}(\phi(n))}} - \sqrt{\frac{2\cdot\log(l_B(\phi(n)))}{t_{y^*_{B_2}}(\phi(n))}}\\
    \intertext{($\downarrow$ since $\bmu_{y_{B_2}^*}>\mu_{y_{B_2}^*}$ by assumption on $\phi$)}\\
    &<\mu_{x_J}-\mu_{y_J}+\mu_{x_{B_2}^*} + 3\delta -\mu_{y_{B_2}^*} \\
    &- \sqrt{\frac{2\cdot\log(l_B(\phi(n)))}{t_{x_J}(\phi(n))}} + \sqrt{\frac{2\cdot\log(l_B(\phi(n)))}{t_{y_J}(\phi(n))}}
    + \sqrt{\frac{2\cdot\log(l_B(\phi(n)))}{t_{x^*_{B_2}}(\phi(n))}}\\
    &<-\mu_{x_J}+\mu_{y_J}+\mu_{x_{B_2}^*}+3\delta-\mu_{y_{B_2}^*} + \sqrt{\frac{2\cdot\log(l_B(\phi(n)))}{t_{y_J}(\phi(n))}} + \sqrt{\frac{2\cdot\log(l_B(\phi(n)))}{t_{x^*_{B_2}}(\phi(n))}}\\
    &=\mu_{y_J}-\mu_{x_J}+\mu_{x_{B_2}^*}-\mu_{y_{B_2}^*}+3\delta + \sqrt{\frac{2\cdot\log(l_B(\phi(n)))}{t_{y_J}(\phi(n))}} + \sqrt{\frac{2\cdot\log(l_B(\phi(n)))}{t_{x^*_{B_2}}(\phi(n))}}\\
\end{align}
But now, since $\mu_{B_2}>\mu_{J}$, one can pick $\delta$ such that $\delta<([\mu_{y_{B_2}^*}-\mu_{x_{B_2}^*}] - [\mu_{y_J} - \mu_{x_J}])/5$. 

By design of \texttt{MAUB}, and by assumption (a), it respectively holds that $t_{x_{B_2}^*}(\phi(n)) \geq l_B(\phi(n))/(\gamma+1) \geq l_B(\phi(n))/2(\gamma+1)$, and $t_{y_J}(\phi(n)) \geq t_J(\phi(n)) \geq t_{B,J}(\phi(n)) \geq l_B(\phi(n))/2(\gamma+1)$, we have:
\begin{align}
    &\sqrt{\frac{2\cdot\log(l_B(\phi(n))}{t_{x_{B_2}^*}}}< \sqrt{\frac{2\cdot\log(l_B(\phi(n))}{l_{B}(n)/(2(\gamma+1)}} < \sqrt{\frac{2\cdot\log(l_B(n))}{l_{B}(n)/2(\gamma+1)}}\text{, and}\\
    &\sqrt{\frac{2\cdot\log(l_B(\phi(n))}{t_{y_{J}^*}}}< \sqrt{\frac{2\cdot\log(l_B(\phi(n))}{l_{B}(n)/2(\gamma+1)}} < \sqrt{\frac{2\cdot\log(l_B(n))}{l_{B}(n)/2(\gamma+1)}}.
\end{align}
But then, $n \mapsto \sqrt{\frac{2\cdot\log(l_B(n))}{l_{B}(n)/2(\gamma+1))}}$ converges to $0$ as $n$ goes to $\infty$. Therefore, there must exist $l_0$ such that $l_B(n) \geq l_0$ implies that $\sqrt{\frac{2\cdot\log(l_B(n))}{l_{B}(n)/2(\gamma+1)}}<\delta$.

\paragraph{(b)}

Now assume that $t_{B,B}\geq (3/2)l_B(n)/(\gamma+1) +1$. 
$B$ and $B_2$ can directly be compared through $x_{B_2}^*$ and $y_{B_2}^*$.
There are at least $l_B(n)/2(\gamma+1)+1$ instants $\tilde n$ such that $B$ was selected normally ({\em i.e.} not forced). 
By the same reasoning as in (a), there exist a $\phi(n)$ such that $\cL(\phi(n))=B$, $\cP(\phi(n))=B$, $t_{B,B}(\phi(n))=\lfloor l_B(n)/2(\gamma+1) \rfloor$ and $(l_B(\phi(n))-1)$ is not a multiple of $1/(\gamma+1)$. Thus, $\vbmu_B(\phi(n))\geq \vbmu_{B_2}(\phi(n))$. Similar derivations as in (a), with $J=B$ produce an absurdity, which finally gives $\phi(n)\in C_\delta \cup D_\delta \cup E^T$.

Let us define $B_{\epsilon,l_0}^T=\{n\leq T~|~n \in B_\epsilon^T, l_B(n) \geq l_0\}$. We have $|B_{\epsilon,l_0}^T|\leq l_0 + |B_\epsilon^T|$. $\phi: n\mapsto\phi(n)$ is a mapping from $B_{\epsilon,l_0}^T$ to $C_\delta \cup D_\delta \cup E^T$. To bound the size of $B_{\epsilon,l_0}^T$, we use the following decomposition:
$$
\{n~|~n\in B_{\epsilon,l_0}^T, l_B(n) \geq l_0\}\subseteq\cup_{n' \in C_\delta \cup D_\delta \cup E^T} \{n~|~n\in B_{\epsilon,l_0}^T,\ \phi(n)=n'\}.
$$
Let us fix $n'$. If $n \in B_{\epsilon,l_0}^T$ and $\phi(n)=n'$, then $\lfloor l_B(n)/2(\gamma+1) \rfloor \in \cup_{J\in\cN_{B,\vmu}\setminus\{B_2\}} \{t_{B,J}(n')\}$ and $l_B(n)$ is incremented at time $n$ because $\cL(n)=B$. 
Therefore:
$$
|\{n~|~n \in B_{\epsilon,l_0}^T,\ \phi(n)=n'\}|\leq 2\gamma(\gamma+1)
$$
Using union bound, we obtain the desired result:
$$
|B_\epsilon^T|\leq l_0+|B_{\epsilon,l_0}^T|\leq \cO(1) + 2\gamma(\gamma+1)(|C_\delta|+|D_\delta|+|E^T|).
$$

\paragraph{Bound on $C_\delta$}

We apply Lem.~\ref{app:lemma:CombesProutiere:B1} with $\Lambda(s)\eqdef\{n\leq T| \cL(n)=B, l_B(n)=s\}$, and $\Lambda\eqdef\cup_{s\geq 1}\Lambda(s)$. It holds that: 
\begin{align}
    C_\delta&=\{n\in\cup_s\Lambda(s)~|~\exists x, |\hmu_x(n) - \mu_x|>\delta\}\\
	    &\subseteq\cup_x\{n\in\cup_s\Lambda(s)~|~|\hmu_x(n)-\mu_x|>\delta\}\\
	    &\eqdef\cup_x C_{\delta,x}.
    \end{align}
    Since each time $B$ is leader, $l_B(n)$ increases, $\forall s,\ |\Lambda(s)|\leq 1$. By design of \texttt{MAUB}, it holds that $t_B(n)\geq1/(\gamma+1) \cdot l_B(n)$, so that $\forall s,\ t_B(s)>\epsilon\cdot s$ for some $\epsilon<1/(\gamma+1)$. 
    Thus, for all $n,x\in C_\delta\cup B$, $t_x(n)\geq \epsilon s$. 
    For all $x\in B$, Lem.~\ref{app:lemma:CombesProutiere:B1} gives $\mathds{E}[|C_{\delta,x}|]=\mathds{E}[\sum_n\mathds{1}_{\{n\in\Lambda~|~|\hmu_x(n) - \mu_x|>\delta\}}]=\cO(1)$.
    We thus get $\mathds{E}[|C_\delta|]\leq \sum_{x\in B} \mathds{E}[|C_{\delta,x}|]=\cO(1)$.

\paragraph{Bound on $D_\delta$}

For any $J\in\cN_{B,\vmu}$, for any $x \in J$ we apply Lem.~\ref{app:lemma:CombesProutiere:B1} with $\Lambda_{J,x}(s)\eqdef\{n\leq T| \cL(n)=B, \cP(n)=J, t_J(n)=s\}$. 
Then:
\begin{align}
    D_\delta \subseteq \cup_J \cup_x \{n\in \Lambda_{J,x}(s)~|~|\hmu_x-\mu_x|>\delta\}
\end{align}
Since each time $B$ is leader and $J$ is played, $t_J(n)$ increases, $\forall s, \forall x, \ |\Lambda_{J,x}(s)|\leq 1$. By definition of $\Lambda_{J,x}$, it holds that $\forall n\in\Lambda(s), t_{x}(n)=s>\epsilon\cdot s$, for some $\epsilon<1$. 
For all $J,x$, Lem.\ref{app:lemma:CombesProutiere:B1} gives $\mathds{E}[\sum_n\{n\in \Lambda_{J,x}(s)~|~|\hmu_x-\mu_x|>\delta\}] = \cO(1)$.
Summing over the finite set $J$, and over the finite neighborhood $\cN_{B,\vmu}$, we finally get $\mathds{E}[|D_{\delta}|]=\cO(1)$.

\paragraph{Bound on $E^T$}

It holds that $E^T=\cO(\log\log(T))$.
For all $x^*$ in $B_2$, and for all $t$ in $\mathbb{N}^*$, it holds that:
\begin{align}
    \mathds{P}(\hmu_{x^*}(t)\leq\mu_{x^*}-\sqrt{\frac{2\log(l_{B}(t))}{N_{x^*}(t)}}) 
    &\leq \exp(-2 \frac{N_{x^*}(t) 2 \log(l_B(t))}{N_{x^*}(t)}) \qquad \text{(Hoeffding's inequality)}\\
    & = \exp(-4\log(l_B(t)))\\
    &=l_B(t)^{-4}.
\end{align}

For all $s$, let $\Lambda(s)\eqdef\{n \leq T~|~l_B(n)=s, \cL(n)=B\}$ be the set of unique time step where $B$ is the leader for the $s$-th time. We define $\Lambda\eqdef\cup_s\Lambda(s)$
For all $s$, let $\{\phi_s\}=\Lambda(s)$ if $\{1,\dots,T\}\cap\Lambda \neq \emptyset$ and $\phi_s=T+1$ otherwise.
\begin{align}
    \label{eq:hoeffding_start}
&\mathds{E}[\sum_{n=1}^T \mathds{1}\{\cL(n)=B, \bmu_{x^*}(n)<\mu_{x^*}\}] \\
&=\mathds{E}[\sum_{n=1}^T \mathds{1}\{n\in\Lambda,\ \bmu_{x^*}(n)<\mu_{x^*}\}] \\
&\leq \mathds{E}[\sum_{s \geq 1}^T \mathds{1}\{ \hmu_{x^*}(\phi_s)+\sqrt{\frac{2\log(l_B(\phi_s))}{t_x(\phi_s)}}<\mu_{x^*}, \phi_s \leq T\}] \\
&\leq \mathds{E}[\sum_{s \geq 1}  \mathds{1}\{ \hmu_{x^*}(\phi_s)+\sqrt{\frac{2\log(s)}{t_x(\phi_s)}}<\mu_{x^*}\}] \\
\label{eq:hoeffding_end}
\end{align}

All this put together gives $\mathds{E}[B_\epsilon^T]=\cO(\log\log(T))$, and adding the bound on $\mathds{E}[A_\epsilon]$ yields $\mathds{E}[l_B(T)]=\cO(\log\log(T))$, which concludes the proof.
\end{proof}

\section{Regret}
\label{app:proof:regretMAUB}

\regretMAUB*

\begin{proof}
    Given Th.~\ref{th:suboptimalArms}, we restrict the study to the iterations for which $B^*$ is the leader. \texttt{MAUB} behaves as a classical \texttt{UCB} on the finite neighborhood of $B^*$. We therefore lightly adapt \citep{aue-ml-02}'s proof to our case. Let $B=B^*-\mu_{\sigma(e)}+\mu_e$.

	Let us upper bound the number of times $t_B(n)$ that a suboptimal arm $B$ was played up to time step $n$. Let $\ell \in \mathbb{N}^*$.
    \begin{align}
	t_B(n)&=\sum_{t=1}^n \mathds{1}_{\{\cP(t)=B\}}\\
	&\leq \ell + \sum_{t=1}^n \mathds{1}_{\{\cP(t)=B,\ t_B(t)\geq\ell\}}.
    \end{align}
    But now, for $B$ to be played at time step $t$, it must hold that $\vbmu_B(t)>\vbmu_{B^*}(t)$, so that the event $\{\cP(t)=B\}$ is a subset of the event $\{\vbmu_B(t)>\vbmu_{B^*}(t)\}$.
    It follows that: 
    \begin{align}
	\mathds{E}[t_B(n)]&\leq \ell + \mathds{E}[\sum_{t=1}^n \mathds{1}_{\{\vbmu_B(t)>\vbmu_{B^*}(t),\ t_B(t)\geq\ell\}}]\\
	      &\leq \ell + \sum_{t=1}^n \mathds{P}(\hmu_{e}(t)+\sqrt{\frac{2\log(l_{B^*}(t))}{N_{e}(t)}}>\hmu_{\sigma(e)}(t)+\sqrt{\frac{2\log(l_{B^*}(t))}{N_{\sigma(e)}(t)}},\ t_B(t)\geq\ell).
    \end{align}
    The event $\{\hmu_{e}(t)+\sqrt{\frac{2\log(l_{B^*}(t))}{N_{e}(t)}}>\hmu_{\sigma(e)}(t)+\sqrt{\frac{2\log(l_{B^*}(t))}{N_{\sigma(e)}(t)}},\ t_B(t)\geq\ell)\}$ is included in the event $\bigcup_{s=1}^n\bigcup_{s_i=\ell}^n\{\hmu_{e}+\sqrt{\frac{2\log(l_{B^*}(t))}{s_i}}>\hmu_{\sigma(e)}+\sqrt{\frac{2\log(l_{B^*}(t))}{s}}\}$.
    Now observe that $\{\hmu_{e}+\sqrt{\frac{2\log(l_{B^*}(t))}{s_i}}>\hmu_{\sigma(e)}+\sqrt{\frac{2\log(l_{B^*}(t))}{s}}\}$ implies that at least one of the following must hold:
    \begin{align}
	\hmu_{\sigma(e)}\leq\mu_{\sigma(e)}-\sqrt{\frac{2\log(l_{B^*}(t))}{s}}\\
	\hmu_{e}\geq\mu_{e}+\sqrt{\frac{2\log(l_{B^*}(t))}{s_i}}\\
	\mu_{\sigma(e)}<\mu_{e}+2\sqrt{\frac{2\log(l_{B^*}(t))}{s_i}}.
    \end{align}

    For $\ell\geq\lfloor 8\log(l_{B^*}(n))/|\mu_{\sigma(e)}-\mu_{e}|^2 \rfloor+1$, the last event can not happen. 

    Both other events are bounded as before (Equations \ref{eq:hoeffding_start} to \ref{eq:hoeffding_end}) using Hoeffding's inequality. It follows that for all $B\in \cN_{B^*,\vmu}$, $t_B(n) = \cO(8\log(n)/|\mu_{\sigma(e)}-\mu_{e}|)$. Summing over $\cN_{B^*,\vmu}$ gives the overall regret $\cO(\sum_{e \notin B^*} \frac{8\cdot\log T}{\mu_{\sigma(e)} - \mu_e})$
\end{proof}

\end{document}